\documentclass[11pt]{article}

\usepackage{parskip}
\usepackage[margin=1in]{geometry}
\usepackage[english]{babel}
\usepackage{microtype}
\usepackage{amsmath,amssymb,amsthm}
\usepackage{bm}
\usepackage{hyperref}
\usepackage{mathrsfs}
\usepackage[utf8]{inputenc} 
\usepackage[T1]{fontenc}    
\usepackage{hyperref}       
\usepackage{url}            % simple URL typesetting
\usepackage{booktabs}       % professional-quality tables
\usepackage{amsfonts}       % blackboard math symbols
\usepackage{nicefrac}      
\usepackage{microtype}      % microtypography
\usepackage{xcolor}         
\usepackage{booktabs}
\usepackage{multirow}
\usepackage{microtype}
\usepackage{authblk}
\usepackage{tikz}
\usetikzlibrary{positioning, arrows.meta, fit, backgrounds, calc, shapes.geometric}
\usepackage{pdfpages}

\usepackage{graphicx}
\usepackage{subcaption}
\usepackage{enumitem}
\usepackage{amsmath,amssymb,amsthm}
\usepackage{bm}

\usepackage{xcolor}
\definecolor{codebg}{rgb}{0.96,0.96,0.96}
\usepackage{algorithm}
\usepackage[noend]{algpseudocode}

\newtheorem{lemma}{Lemma}
\newtheorem{proposition}{Proposition}
\newtheorem{corollary}{Corollary}
\newtheorem{remark}{Remark}
\newtheorem{example}{Example}
\usepackage{natbib}

\title{Soft Specialists: $\alpha$-Rényi Ensembles for Uncertainty-Aware LLM Post-Training}

\author[1]{Paula Cordero-Encinar}
\author[1]{Georgy Tyukin}
\author[1,2]{Andrew B. Duncan}

\affil[1]{Department of Mathematics, Imperial College London}
\affil[2]{Bessemer AI}

\begin{document}

\maketitle

\begin{abstract}
 Existing training approaches for large language models learn a single set of parameters, based on large volumes of data, which is typically heterogeneous, conflicting and often outright contradictory.   As a result, the model is forced to compress conflicting goals, and inherent uncertainties into a single, averaged pattern of behaviour.   We propose an $\alpha$-R\'{e}nyi variational framework for learning distributions over post-training parameters, offering an uncertainty-aware alternative to  deep ensemble approaches. The resulting variational objective interpolates between classical variational Bayes and predictively oriented posterior learning,  balancing between globally plausible individual models against systems of complementary specialists. We identify local stability criteria,   demonstrating how model misspecification can make non-degenerate posterior spread locally favourable,  manifesting contradictory or conflicting data as epistemic uncertainty.  We apply our framework to LLM post-training,  learning an ensemble of LoRA adapters attached to a shared, frozen base model, providing a scalable training procedure for both supervised fine-tuning and preference optimisation. Our approach enables training examples to be softly routed across ensemble members, promoting model specialisation and providing actionable uncertainty estimates across different tasks.
\end{abstract}

\section{Introduction}
\label{sec:intro}

Large language models are typically adapted to downstream tasks, safety requirements and user preferences through post-training techniques such as Supervised Fine-Tuning (SFT) \cite{qi2024fine}, Direct Preference Optimisation (DPO), \cite{rafailov2023direct} or Reinforcement Learning from Human Feedback (RLHF) \cite{bai2022training, dai2024safe, ouyang2022training}.  In all these approaches, the output is a single adapted parameter vector, which must compress any uncertainty, ambiguity as well as and potentially conflicting preferences.   This forces the learned parameter vector to absorb competing pressures that may not admit a satisfactory joint resolution.

This issue has implications for alignment and high-stakes deployment. Ideally, a single post-trained model is expected to preserve useful base capabilities, remaining helpful under ambiguous requests while avoiding harmful behaviour \cite{anwar2024foundational, ganguli2022red}. However, these goals need not be simultaneously well represented by a single set of parameters: a model pushed toward safer operation may over-refuse benign requests or degrade capability, while a model tuned to remain helpful may become brittle under ambiguous or malicious prompts \cite{rottger2023xstest, cui2025or}. In essence, a single adapter post-training pipeline has no explicit mechanism for representing unresolved uncertainty.

Bayesian methods offer one potential approach through their ability to quantify epistemic uncertainty.  Bayesian neural networks, variational inference, Laplace approximations, dropout-based approximations, and sampling-based methods have all been used to represent epistemic uncertainty in deep learning \citep{arbel2026primer,kingma2013auto,blundell2015weight,kingma2015variational,mackay1992practical,yang2024bayesian,neal2012bayesian,chen2014stochastic,gal2016dropout}.  Bayesian methods have been deployed in the context of Language model safety, \cite{hu2026adaptive}.     However, the posterior distribution is intended to represent epistemic uncertainty over parameter values that are {individually} plausible explanations of the entirety of the observed data under the assumed model and prior.  A ramification of this is that even mild forms of model misspecification can cause Bayesian posterior predictive distributions in performing arbitrarily badly \citep{grunwald2017inconsistency}.

Deep ensembles \citep{fort2019deep, lakshminarayanan2017simple} offer a practical second approach.  
By combining  independently trained models at inference time, they often yield useful heuristic uncertainty,  particularly when combined with post-hoc calibration strategies \cite{angelopoulos2025learn,rivera2024conformal}. As their diversity is usually induced through random initialisation, data order and optimisation noise, ensemble members are not coordinated as samples of a posterior distribution.   
However, introducing appropriate coordination during training allows these ensembles to effectively approximate a Bayesian posterior \cite{d2021repulsive, wild2023rigorous}.

A third approach arises from generalised Bayesian methods \cite{grunwald2007suboptimal, zellner1988optimal, aitchison1975goodness, bissiri2016general,Knoblauch2022RoT}, designed to better handle misspecification.   Methods such as Gibbs posteriors \citep{jiang2008gibbs, martin2022direct}, tempered or fractional posteriors \citep{bhattacharya2019bayesian,wenzel2020good}, and safe-Bayesian approaches \citep{grunwald2012safe}   adapt posterior concentration to mitigate brittleness arising from the mismatch between the (unseen) data-generating distribution and the parametrised model family. These methods are highly relevant to  training large-scale deep learning models, where the distribution is rarely clean or  homogeneous.   Nevertheless, they retain the same basic structure where posterior mass is assigned according to the averaged loss of individual parameter choices over the full dataset. 

Recent work on predictively oriented posteriors (PrO) \citep{McLatchie2025PrO} builds upon the generalised Bayesian approach, by evaluating a distribution \(Q\) through the predictive quality of the mixture it induces.
This is effective under heterogeneity or misspecification: no single parameter may adequately explain the data, but a mixture of complementary predictors might. However, the predictively oriented objective lies at the opposite end of the spectrum from generalised Bayesian inference. In practice, one may want both: individually meaningful and regularised parameter settings, but also enough predictive cooperation to represent conflicting or heterogeneous data. {This motivates the need for a principled approach to interpolating between classical and prediction oriented posteriors}.

In this paper, we introduce an \(\alpha\)-R\'enyi variational framework for learning distributions over post-training parameters. For a distribution \(Q\) over parameters and a labelled example \((x,y)\), we define the per-example loss
\begin{equation*}
   \ell_\alpha(Q;x,y)
=
-\frac{1}{\alpha}
\log
\int_\Theta p_\theta(y\mid x)^\alpha\,Q(d\theta),
\qquad
\alpha\in(0,1], 
\end{equation*}
where $p_{\theta}(y \, | \,x)$ is the conditional likelihood assigned by the model with parameters \(\theta\).  As \(\alpha\to 0\), this recovers the classical variational Bayes data-fit term $-\mathbb E_{\theta\sim Q}\big[\log p_\theta(y\mid x)\big]$, while at \(\alpha=1\) it becomes the negative log-likelihood of the predictive mixture.
The parameter \(\alpha\)  directly influences the geometry of the variational objective: small \(\alpha\) favours globally plausible individual models, while larger \(\alpha\) increasingly rewards complementary specialists whose utility emerges through aggregation.

Our main application is distributional LLM post-training: instead of compressing heterogeneous or conflicting supervision into a single adapter, we learn a structured ensemble of LoRA adapters \cite{hu2022lora} whose diversity represents unresolved epistemic uncertainty.  Our parameter space $\Theta$ characterises low rank perturbations of a common  frozen base model.  The resulting  objective induces soft responsibilities across particles where each training example is routed towards those adapters that currently explain it well. In this way, specialisation emerges naturally from the variational objective itself.

\begin{figure}[t]
    \centering
    \includegraphics[width=0.8\linewidth]{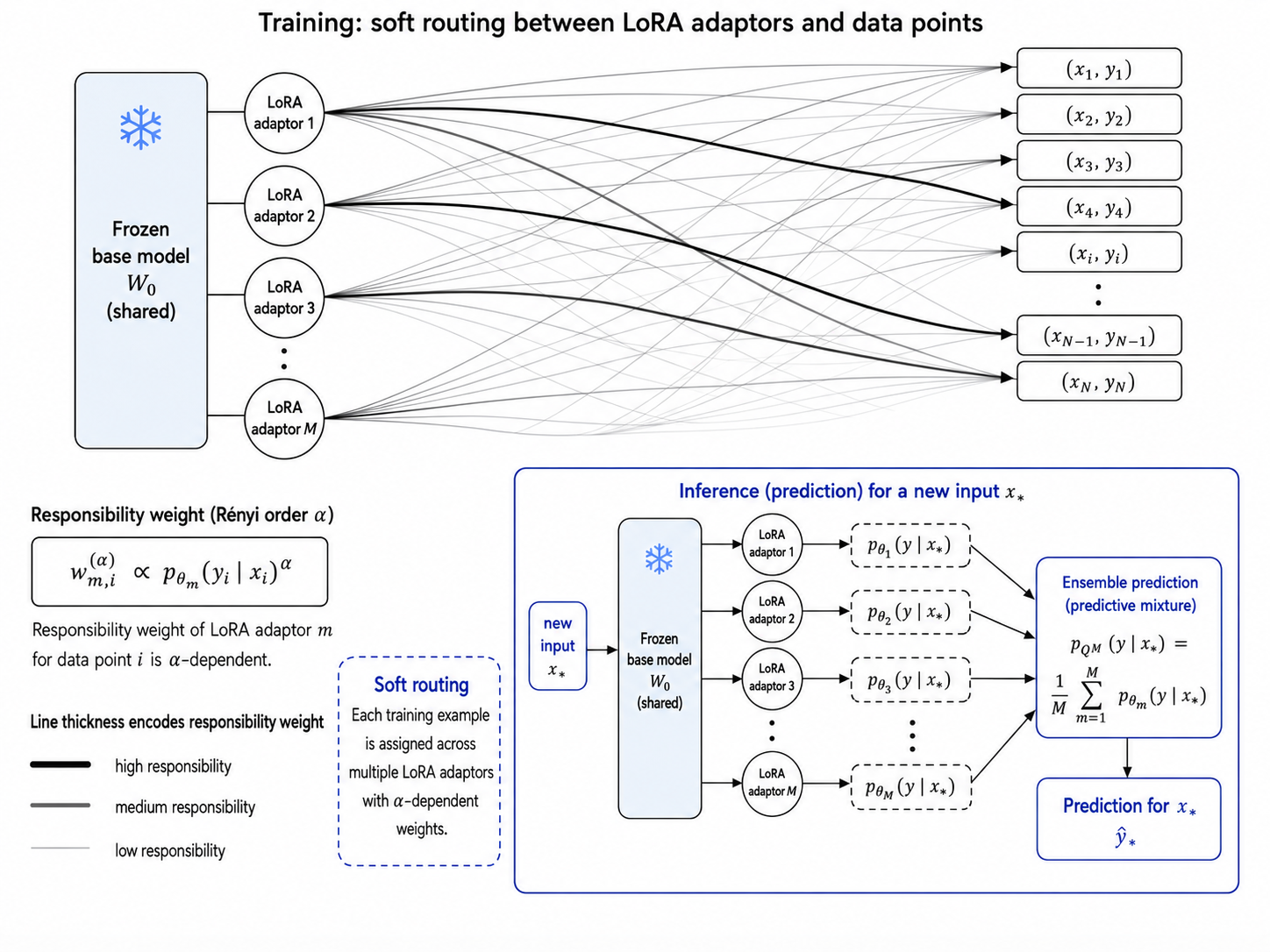}
    \caption{$\alpha$-R\'enyi flow training for LLM ensembles.  A single frozen base model $W_0$ is shared across $M$ trainable LoRA particles.  For each minibatch example, the particles produce sequence log-likelihoods $s_{i,b}$ which are coupled through the objective.  The resulting responsibilities $w_{i,b}^{(\alpha)}$ softly route examples towards particles that explain them well, inducing specialisation for $\alpha > 0$.  The final predictor is the induced mixture $p_{Q^M}$.} 
    \label{fig:lora_alpha_renyi_ensemble}
\end{figure}

Our contributions are:
\begin{enumerate}[leftmargin=1.4em]
    \item We introduce an $\alpha$-indexed variational objective over distributions of post-training parameters, interpolating between classical variational Bayes and predictively oriented posterior learning.
    \item We analyse the local stability of Dirac posteriors and show how misspecification can make non-degenerate posterior spread favourable. This behaviour is governed by an information-geometric  matrix that underpins the variational objective.
    \item We derive finite-particle training objectives for supervised fine-tuning and preference optimisation, yielding soft responsibility-based routing across LoRA adapters in LLMs, see Figure \ref{fig:lora_alpha_renyi_ensemble}.
    \item We demonstrate how the resulting ensemble can represent localised epistemic uncertainty under misspecification and contamination.
\end{enumerate}
We provide a more in-depth discussion of related work in App.~\ref{app:related_work}.

\section{Setup}
\label{sec:framework}

We begin by recalling the variational formulations underlying classical Bayesian inference and predictively oriented posteriors, and then introduce the $\alpha$-indexed family of objectives studied in this paper. Throughout, our goal is to make explicit the distinction between posteriors that are evaluated through the individual quality of their constituent parameter settings, and posteriors that are evaluated through the predictive quality of the mixture they induce.

\subsection{Interpolating between classical and predictively oriented posteriors}

Let $\Theta \subseteq \mathbb{R}^d$ denote a parameter space, and let $\{p_\theta(y \mid x) : \theta \in \Theta\}$ be a parametric family of conditional predictive models. Consider a dataset $\mathcal{D} = \{(x_n,y_n)\}_{n=1}^N$,
of independent samples from an unknown data-generating distribution $\nu$ on pairs $(x,y)$, and let $\pi_0,\, Q \in \mathcal{P}(\Theta)$ denote the reference prior and candidate posterior distributions, respectively.

The key object of interest throughout is the predictive distribution induced by $Q$, defined as the mixture predictor 
\begin{equation}
p_Q(y \mid x)
:=
\int_\Theta p_\theta(y \mid x)\,Q(d\theta).
\label{eq:mixture_predictive}
\end{equation}
When \(Q\) is the Bayesian posterior \(\pi(\cdot\mid\mathcal D)\), the distribution \(p_Q\) becomes the usual Bayesian posterior predictive, thus recovering Bayesian model averaging \cite{masegosa2020learning}.  

At the population level, the predictive target is the mixture cross-entropy, that is, $$Q^* \in \arg\min_{Q}\mathbb{E}_{(x,y)\sim \nu}[-\log \mathbb{E}_{\theta \sim Q}[p_\theta(y \mid x)]],$$ based on observations in $\mathcal{D}$, that generalise well to unseen data.
As defined, there is freedom in how we interpret $Q$, e.g. as a posterior distribution over {individually capable} predictors, or as a predictive mixture, whose aggregate yields a capable predictor. Generalised Bayesian  and predictively oriented posteriors emphasise these two viewpoints respectively.

\paragraph{Bayesian posteriors}

In the standard Bayesian formulation, the posterior over parameters is given formally by
$
\pi(d\theta \mid \mathcal{D})
\propto
p_\theta(\mathcal{D})\,\pi_0(d\theta)$, where
$p_\theta(\mathcal{D})=\prod_{n=1}^N p_\theta(y_n \mid x_n)
$.
When the exact posterior is intractable, one typically introduces a variational family and chooses $Q$ to minimise the Kullback-Leibler divergence to the true posterior. Equivalently, one minimises the free energy functional
\begin{equation}
\mathcal{F}_{\mathrm{B}}(Q)
=
-\sum_{n=1}^N \mathbb{E}_{\theta \sim Q}\big[\log p_\theta(y_n \mid x_n)\big]
+
\mbox{KL}(Q \,\|\, \pi_0),
\label{eq:vb_objective}
\end{equation}
over  $Q$.
Up to scaling conventions, this is the usual evidence lower bound (ELBO) objective.   The data-fit term in \eqref{eq:vb_objective} is the posterior expectation of the negative log-likelihood, $\sum_{n=1}^N \mathbb{E}_Q[-\log p_\theta(y_n \mid x_n)]$. 
More generally, a broad class of generalised Bayesian or Gibbs posterior methods replaces the log-likelihood by an arbitrary loss function $\ell(\theta; x,y)$, yielding objectives of the form
\begin{equation}
\mathcal{F}_{\mathrm{GB}}(Q)
=
\sum_{n=1}^N \mathbb{E}_{\theta \sim Q}\big[\ell(\theta; x_n,y_n)\big]
+
\lambda\,\mbox{KL}(Q \,\|\, \pi_0),
\label{eq:gb_objective}
\end{equation}
for some learning-rate or regularisation parameter $\lambda > 0$. This formulation encompasses tempered, fractional, and other generalised Bayesian procedures.  For both the classical and generalised Bayes objectives  \eqref{eq:vb_objective} and \eqref{eq:gb_objective}, the associated data-fit terms are linear in $Q$.  Minimising this term results in a distribution $Q$ which concentrates on $\theta$ values which minimise $\theta \rightarrow \sum_{i=1}^N l(\theta ; x_n, y_n)$ which measures a parameters individual predictive capability over the entire dataset.   As a result, these variational objectives score each parameter by its overall performance over the full dataset. In the large-data or weak-regularisation regime, this favours parameter values that act as global generalists rather than specialists whose value appears only through mixture aggregation.

Note that the Bayesian data-fit term  minimises the expected log-loss of individual models, given by $\mathbb{E}_{(x,y)\sim \nu}[\mathbb{E}_{\theta \sim Q}[-\log p_\theta(y \mid x)]]$, rather than the cross-entropy of the predictive mixture. 
By Jensen's inequality, this yields an upper bound on the true predictive risk
\begin{equation}
\mathbb{E}_{(x,y)\sim\nu}\big[-\log \mathbb{E}_{\theta \sim Q}[p_\theta(y \mid x)]\big] 
\leq 
\mathbb{E}_{(x,y)\sim\nu}\big[\mathbb{E}_{\theta \sim Q}[-\log p_\theta(y \mid x)]\big].
\label{eq:jensens_gap}
\end{equation}

In the well-specified case, where \(p^\star(\cdot\mid x)=p_{\theta^\star}(\cdot\mid x)\) for \(P_X^\star\)-a.e. \(x\), these two objectives have the same predictive optimum. Under misspecification, however, the optima can differ: the expected individual log-loss selects a single generalist, whereas the mixture cross-entropy may be minimised by a non-degenerate mixture of complementary predictors, \cite[Lemma 2]{masegosa2020learning}.

\paragraph{Predictively oriented posteriors}

Predictively oriented posteriors take a different starting point. Rather than defining a posterior over parameters through the average fit of its individual members, the idea is to select $Q$ based on the cross-entropy loss induced by \eqref{eq:mixture_predictive}.  This yields the objective
\begin{equation}
\mathcal{F}_{\mathrm{PrO}}(Q)
=
-\sum_{n=1}^N \log p_Q(y_n \mid x_n)
+
\lambda\,\mbox{KL}(Q \,\|\, \pi_0),
\label{eq:pro_objective}
\end{equation}
where $p_Q(y_n \mid x_n) = \int_\Theta p_\theta(y_n \mid x_n)\,Q(d\theta)$.
Unlike the Bayesian objectives \eqref{eq:vb_objective} and \eqref{eq:gb_objective}, the data-fit term in \eqref{eq:pro_objective} is no longer linear in $Q$, and in particular, a minimiser $Q$ will concentrate onto parameters based on which may be weak across the full dataset, but jointly capable as a mixture, thus promoting specialisation.  Similar ideas have been previously studied, e.g. \citep{lai2024predictive, lacasse2006pac, masegosa2020learning}. 

This distinction is especially important under data heterogeneity or model misspecification. When no single parameter value offers a satisfactory global explanation of the data, a mixture of complementary parameter settings can nevertheless yield a significantly stronger predictive model. Predictively oriented posteriors are designed to capture this phenomenon.

\paragraph{Defining $\alpha$-R\'enyi posteriors}

Classical variational Bayes and predictively oriented posteriors should therefore be understood as emphasising two different inferential viewpoints.  Our goal is to interpolate between these two endpoints in a principled manner.
To this end, for each $\alpha \in (0,1]$ we define the per-example loss
\begin{equation}
\ell_\alpha(Q; x,y)
:=
-\frac{1}{\alpha}
\log
\int_\Theta p_\theta(y \mid x)^\alpha\,Q(d\theta)\,.
\label{eq:alpha_pointwise}
\end{equation}
The corresponding variational objective is
\begin{equation}
\mathcal{F}_\alpha(Q)
=
\sum_{n=1}^N \ell_\alpha(Q; x_n,y_n)
+
\lambda\,\mbox{KL}(Q \,\|\, \pi_0).
\label{eq:alpha_objective}
\end{equation}
This family interpolates continuously between the classical and predictively oriented objectives. Indeed, for fixed $(x,y)$ and sufficiently integrable likelihoods, $\lim_{\alpha \to 0} \ell_\alpha(Q; x,y)
=
-\mathbb{E}_{\theta \sim Q}[\log p_\theta(y \mid x)]$,
whereas at $\alpha=1$, we have that $\ell_1(Q; x,y)
=-\log p_Q(y \mid x)$.
Consequently, \(\mathcal F_\alpha\) recovers the log-loss variational Bayes/Gibbs objective as \(\alpha\to0\), and \(\mathcal F_1=\mathcal F_{\mathrm{PrO}}\). In this sense, $\alpha$ is an \emph{inferential interpolation parameter} which determines the extent to which one promotes individually plausible models versus complementary predictors whose utility emerges in combination.

Objective \eqref{eq:alpha_objective} provides a natural unified variational framework spanning classical and predictively oriented posteriors. 
It preserves the regularised variational form familiar from Bayesian and generalised Bayesian inference, while providing a principled and interpretable path between the two endpoints.

\section{The $\alpha$-R\'enyi variational framework}

\subsection{Properties of the $\alpha$-R\'enyi loss}

We first study the proposed \(\alpha\)-R\'enyi loss \(\ell_{\alpha}\) and its underlying geometry. A defining structural property of this variational form is its non-linear aggregation of predictions across the distribution $Q$.

We first formalise how the parameter \(\alpha\) induces a smooth interpolation  between classical parameter-based variational inference and predictive mixture training. The proof is provided in App. \ref{app:Proofs}.

\begin{lemma}[Limits, interpolation, and monotonicity of the \texorpdfstring{\(\alpha\)}{alpha}-loss]
\label{lem:alpha-interpolation}
Fix a data point \((x,y)\) and a distribution \(Q\in\mathcal P(\Theta)\). Assume that
\[
p_\theta(y\mid x)>0
\qquad
\text{for \(Q\)-a.e. }\theta,
\]
and that \(\log p_\theta(y\mid x)\) and \(p_\theta(y\mid x)^\alpha\) are
\(Q\)-integrable for the values of \(\alpha\) under consideration. Then the following hold.

\begin{enumerate}
    \item (\emph{Variational objective limit as \(\alpha\to0\)}) As \(\alpha\to0\),
    \[
    \ell_0(Q;x,y)
    :=
    \lim_{\alpha\to0}\ell_\alpha(Q;x,y)
    =
    -\mathbb E_{\theta\sim Q}
    \big[\log p_\theta(y\mid x)\big].
    \]

    \item (\emph{Predictive mixture at \(\alpha=1\)}) At \(\alpha=1\),
    \[
    \ell_1(Q;x,y)
    =
    -\log \mathbb E_{\theta\sim Q}
    \big[p_\theta(y\mid x)\big]
    =
    -\log p_Q(y\mid x).
    \]

    \item (\emph{Interpolation inequality}) For every \(\alpha\in(0,1]\),
    \[
    -\log \mathbb E_{\theta\sim Q}
    \big[p_\theta(y\mid x)\big]
    \;\le\;
    \ell_\alpha(Q;x,y)
    \;\le\;
    -\mathbb E_{\theta\sim Q}
    \big[\log p_\theta(y\mid x)\big].
    \]

    \item (\emph{Monotonicity in \(\alpha\)}) For fixed $Q \in \mathcal{P}(\Theta)$, $(x,y) \in \mathcal{D}$, the map
    \[
    \alpha\mapsto \ell_\alpha(Q;x,y)
    \]
    is non-increasing on \((0,\infty)\). In particular, for
    \(0<\alpha_1\le \alpha_2\le1\),
    \[
    \ell_{\alpha_1}(Q;x,y)
    \ge
    \ell_{\alpha_2}(Q;x,y).
    \]
    The inequality is strict unless \(p_\theta(y\mid x)\) is \(Q\)-a.s. constant.
\end{enumerate}
\end{lemma}

\paragraph{Connection to entropic risk.}
The loss \eqref{eq:alpha_pointwise} can also be viewed as an entropic risk
functional applied to the random loss induced by sampling a parameter from
\(Q\). Writing
\[
L_\theta(x,y):=-\log p_\theta(y\mid x),
\]
we have
\[
\ell_\alpha(Q;x,y)
=
-\frac1\alpha
\log
\mathbb E_{\theta\sim Q}
\left[
\exp\{-\alpha L_\theta(x,y)\}
\right].
\]
Thus \(\ell_\alpha\) is the entropic risk of \(L_\theta(x,y)\) with negative
temperature parameter \(t=-\alpha\). In the usual risk-sensitive convention,
positive temperature emphasises high-loss outcomes, whereas negative
temperature emphasises low-loss outcomes. Our setting therefore corresponds to a risk-seeking entropic transform over parameter draws, where the loss of an example can be reduced when some members of the ensemble explain it particularly well. This is precisely the mechanism which promotes specialisation within the $\alpha$-Renyi ensemble. Note that, unlike tilted empirical risk methods, which tilt across data points for a single model, our objective does so across parameter values for each data point, thus inducing observation-specific routing within the ensemble
\citep{follmer2011entropic,pichler2020entropy,litilted}.

Lemma~\ref{lem:alpha-interpolation} shows that the family
\(\{\mathcal{F}_\alpha\}_{\alpha\in[0,1]}\) continuously connects a
parameter-posterior variational objective to a predictively oriented posterior
objective. The entropic-risk interpretation above clarifies the role of
\(\alpha\): increasing \(\alpha\) strengthens the negative-temperature tilt
towards parameter values that explain the current observation well. Following
\citep{masegosa2020learning}, we can quantify this effect through a variance
lower bound on the gap between \(\ell_0\) and \(\ell_\alpha\).

\begin{lemma}[Variance lower bound]
\label{lem:alpha_jensen_gap_variance_bound}
Let  \((x,y) \in \mathcal{D}\),  \(Q\in\mathcal P(\Theta)\), and
\(\alpha\in(0,1]\). Assume that
\[
0 < p_\theta(y\mid x) \le M_{x,y} < \infty
\qquad
\text{for \(Q\)-a.e. }\theta,
\]
where $M_{x,y}
:= \operatorname*{ess\,sup}_{\theta\sim Q}
p_\theta(y\mid x)$. Then
$$
\ell_0(Q;x,y)-\ell_\alpha(Q;x,y)
\ge
\frac{1}{2\alpha M_{x,y}^{2\alpha}}
\operatorname{Var}_{\theta\sim Q}
\!\left(
p_\theta(y\mid x)^\alpha
\right).
$$
\end{lemma}

Lemma \ref{lem:alpha_jensen_gap_variance_bound} shows that $\alpha$-R\'enyi loss can only substantially improve over the classical expected log-loss when different parts of the support of $Q$ can explain the observation $(x,y)$ differently. We can see this gap as  quantifying the local value of
specialisation: small when all particles behave similarly, and large when the ensemble contains complementary predictors.

\paragraph{Soft-routing responsibilities} To better understand how $\ell_{\alpha}$ induces different behaviour when $\alpha > 0$, it is convenient to rephrase the loss $\ell_{\alpha}$ in terms of  observation-specific \emph{responsibilities}.   Intuitively, these can be interpreted as governing how learning signals are distributed among different parameter settings.
For a given data point \((x,y)\) and distribution \(Q\), we define the \(\alpha\)-responsibility of a parameter \(\theta\)  as
\begin{equation}
w^{(\alpha)}(\theta; x, y, Q) :=  \frac{p_\theta(y \mid x)^\alpha}{\int_\Theta p_{\theta'}(y \mid x)^\alpha\,Q(d\theta')}.
\label{eq:continuous_responsibilities_def}
\end{equation}
 For a given observation \((x_i, y_i)\), we can write the per-example likelihood through the Donsker-Varadhan variational formula \cite{donsker1976asymptotic,zellner1988optimal},
\begin{equation*}
\ell_\alpha(Q;x_i, y_i)
=
\inf_{R_i \ll Q}
\left\{
\mathbb{E}_{R_i}[-\log p_\theta(y_i \mid x_i)]
+
\frac{1}{\alpha}\operatorname{KL}(R_i\|Q)
\right\}.
\end{equation*}
Therefore, the finite-sample objective may be written as
\begin{equation*}
\mathcal F_{\alpha}(Q)
=
\sum_{i=1}^N
\inf_{R_i \ll Q}
\left\{
\mathbb{E}_{R_i}[-\log p_\theta(y_i \mid x_i)]
+
\frac{1}{\alpha}\operatorname{KL}(R_i\|Q)
\right\}
+
\lambda\,\operatorname{KL}(Q\|\pi_0).
\label{eq:finite_sample_nested}
\end{equation*}

Writing \(r_i=dR_i/dQ\), the Donsker--Varadhan representation becomes
\[
\ell_\alpha(Q;x_i, y_i)
=
\inf_{\substack{r_i\ge 0\\ \mathbb E_Q r_i=1}}
\left\{
-\mathbb E_Q[r_i(\theta)\log p_{\theta}(y_i \mid x_i)]
+
\frac1\alpha
\mathbb E_Q[r_i(\theta)\log r_i(\theta)]
\right\},
\]
it follows that the minimiser is  $r_i^\star(\theta)
= w^{(\alpha)}(\theta; x_i, y_i, Q)$.  Thus, the responsibilities in \eqref{eq:continuous_responsibilities_def} are precisely the
Radon-Nikodym derivatives of the observation-specific tilted measures
\(R_i^\star\) with respect to the global posterior \(Q\). Evaluating the
Donsker-Varadhan representation at the optimiser gives
\begin{equation}
\ell_\alpha(Q;x_i,y_i)
=
-\mathbb E_Q\!\left[
w_i^{(\alpha)}(\theta)\log p_{\theta}(y_i\mid x_i)
\right]
+
\frac{1}{\alpha}
\mathbb E_Q\!\left[
w_i^{(\alpha)}(\theta)\log w_i^{(\alpha)}(\theta)
\right].
\label{eq:donsker_varadhan_optimiser}
\end{equation}

Eq.~\eqref{eq:donsker_varadhan_optimiser} shows that each observation
induces a trade-off between local predictive fit and deviation from \(Q\). The first term favours tilted measures \(R_i\) that place mass
on parameters assigning high likelihood to \((x_i,y_i)\). The second term,
\(\alpha^{-1}\operatorname{KL}(R_i\|Q)\), penalises moving this observation-specific
tilt away from \(Q\). Thus \(\alpha\) controls the strength of local reweighting:
as \(\alpha\to0\), the KL penalty dominates and \(R_i^\star\) remains close to
\(Q\), recovering the classical averaged-loss regime; for larger \(\alpha\),
the penalty weakens and \(R_i^\star\) can concentrate on the parts of \(Q\) that
explain the observation well. This is the variational origin of the soft-routing
responsibilities.

To shed some more light on the behaviour of the R\'enyi variational objective for $0 < \alpha < 1$,  we note that for a fixed $(x,y)$ the responsibility in \eqref{eq:continuous_responsibilities_def} defines a probability density with respect to  $Q(d\theta)$ for every $\alpha$. 
Denote $Z=\int_\Theta p_{\theta'}(y\mid x)\,Q(d\theta')$ and $w(\theta)=w^{(1)}(\theta; x, y, Q)$, we have the following auxiliary expression
$$
p_\theta(y\mid x)^\alpha = (w^{(1)}(\theta;x, y, Q) Z)^\alpha = (w(\theta) Z)^\alpha.
$$
Using this, we can write
\begin{align*}
\ell_\alpha(x,y) &= -\frac{1}{\alpha}\log\int_{\Theta} Z^\alpha w(\theta)^\alpha Q(d\theta) = -\log Z - \frac{1}{\alpha}\log\int_{\Theta} w(\theta)^\alpha Q(d\theta) \\
&= \ell_1(x,y) - \frac{1-\alpha}{\alpha} H_\alpha(w; Q),
\end{align*}
where 
$$
H_\alpha(w; Q) = \frac{1}{1-\alpha} \log \int_{\Theta} w(\theta)^\alpha Q(\theta),
$$
is the R\'enyi entropy of order $\alpha$ of the responsibility distribution with respect to $Q(d\theta)$ instead of the Lebesgue measure. 
Since $H_{\alpha}(w; Q)$ is negative, it follows that $\ell_{\alpha} \geq \ell_1$. The entropy term vanishes only when the responsibilities $w$ are uniform, and is largest when the responsibilities collapse to a point mass, thereby explicitly penalising responsibility collapse. This introduces a competing pressure against the predictive mixture negative log-likelihood term $\ell_1$ which is seeking to make hard assignments of data points to parameters selected by $Q$.  Consequently, intermediate values $0 < \alpha < 1$ add an entropic regularisation effect on the responsibilities,  promoting soft specialisation.

This distinction is especially important for autoregressive language modelling.  In that case
\[
p_{\theta_i}(y\mid x)
=
\exp s_i(x,y),
\qquad
s_i(x,y)
=
\sum_{t=1}^T
\log p_{\theta_i}(y_t\mid x,y_{<t}),
\]
so the responsibilities take the form
\[
w_i^{(\alpha)}(x,y)
=
\frac{\exp(\alpha s_i(x,y))}
{\sum_{j=1}^M \exp(\alpha s_j(x,y))}.
\]
Since the sequence log-likelihood \(s_i(x,y)\) scales with the response length \(T\), even small per-token likelihood differences can produce highly concentrated responsibilities when \(\alpha=1\).  Intermediate values of \(\alpha\) therefore act as a temperature on sequence-level routing, controlling the degree of specialisation across model weights within the ensemble.

\subsection{Properties of the $\alpha$-R\'enyi posterior}

To understand how the $\alpha$-R\'enyi loss influences posterior selection, we can examine the optimal finite-sample posterior. For the Bayesian variational objective \eqref{eq:gb_objective}, i.e. in the $\alpha\rightarrow 0$ limit, admits a unique minimiser which has the usual product form
$$ 
Q^*(\theta) \propto \prod_{i=1}^N p_{\theta}(y_i \mid x_i)^\frac{1}{\lambda} \pi_0(d\theta),
$$
so that each parameter $\theta$ is weighted depending on its individual performance across all the data.  Increasing $\alpha > 0$ introduces a nonlinear coupling over the parameter space, as formalised below.   We first establish convexity of \eqref{eq:alpha_objective}.

\begin{lemma}[Convexity of the \texorpdfstring{\(\alpha\)}{alpha}-Rényi variational objective]
\label{lem:convexity_alpha_objective}
Fix \(\alpha>0\). For each data point \(z=(x,y)\), define
\[
A_z(Q)
:=
\int_\Theta p_\theta(y\mid x)^\alpha\,Q(d\theta),
\qquad
\ell_\alpha(Q;z)
=
-\frac1\alpha\log A_z(Q),
\]
on the domain where \(0<A_z(Q)<\infty\). Then \(Q\mapsto \ell_\alpha(Q;z)\) is convex on \(\mathcal P(\Theta)\).  Consequently, for a fixed dataset \(\mathcal D=\{z_i\}_{i=1}^N\), the functional \eqref{eq:alpha_objective} is convex on \(\{Q\ll\pi_0:\operatorname{KL}(Q\|\pi_0)<\infty\}\). In particular, if \(\lambda>0\), then
\(\mathcal F_{\alpha}\) is strongly convex with respect to total variation on this domain. 
\end{lemma}
An immediate consequence of Lemma \ref{lem:convexity_alpha_objective} is the uniqueness of the minimiser of \eqref{eq:alpha_objective}, if it exists, and its characterisation through a self-consistency equation.

\begin{proposition}[Self-consistent \texorpdfstring{\(\alpha\)}{alpha}-posterior]
\label{prop:self_consistency_alpha}
Let \(\mathcal D=\{z_i\}_{i=1}^N\), with \(z_i=(x_i,y_i)\), be a fixed dataset.
Assume \(\lambda>0\), and consider $\mathcal{F}_{\alpha}$ defined by \eqref{eq:alpha_objective} over \(Q\ll\pi_0\). Suppose that \(\mathcal F_{\alpha}\) admits an interior
minimiser \(Q^\star_\alpha\) with density
\[
q^\star_\alpha(\theta)
:=
\frac{dQ^\star_\alpha}{d\pi_0}(\theta),
\qquad
q^\star_\alpha(\theta)>0
\quad \pi_0\text{-a.e.}
\]
Then \(Q^\star_\alpha\) is the unique minimiser of \(\mathcal F_{\alpha}\).
Moreover, its density satisfies the self-consistency equation
\begin{equation}
\label{eq:self_consistency_alpha}
q^\star_\alpha(\theta)
=
\frac{1}{Z_\alpha}
\exp\!\left(
\frac{1}{\lambda\alpha}
\sum_{i=1}^N
w_i^{(\alpha)}(\theta;Q^\star_\alpha)
\right),
\end{equation}
where \(Z_\alpha\) is the normalising constant and
\[
w_i^{(\alpha)}(\theta;Q^\star_\alpha)
:=
\frac{
p_\theta(y_i\mid x_i)^\alpha
}{
\int_\Theta
p_{\vartheta}(y_i\mid x_i)^\alpha
Q^\star_\alpha(d\vartheta)
}.
\]

Conversely, any strictly positive density \(q\) satisfying this fixed-point
equation and having finite objective value is the unique minimiser.
\end{proposition}

Note that we do not establish existence of this minimiser, which would require establishing suitable compactness/coercivity/lower semi-continuity assumptions.

From \eqref{eq:self_consistency_alpha}, we can see that each observation  induces a local reweighting of the current posterior toward those parameter settings that explain it well.  The final posterior  reconciles all such local predictive tilts simultaneously. Thus $Q_{\alpha}^\star$ is best understood as a \emph{self-consistent predictive posterior}, in contrast to the generalised Bayes posterior which concentrates over individually plausible parameter values.  To understand the difference arising from this nonlinear relationship for $Q$ in \eqref{eq:self_consistency_alpha} we consider the following illustrative example.

\begin{example}
\label{ex:posteriors}
Consider a finite parameter space $\Theta=\{g,a,b\}$, with a uniform prior \(\pi_0\). Suppose we have a dataset comprising two observations
\(z_1,z_2\), with associated likelihoods 
\[
\begin{array}{c|cc}
 & z_1 & z_2 \\
\hline
g & m & m \\
a & h & \varepsilon \\
b & \varepsilon & h
\end{array}
\qquad
0<\varepsilon<m<h<1.
\]
Thus we can consider the model associated with  \(g\) is a generalist, while \(a\) and \(b\) are specialists for
\(z_1\) and \(z_2\), respectively.  Let \(Q\) assign probabilities
\[
Q(g)=q_g,
\qquad
Q(a)=q_a,
\qquad
Q(b)=q_b.
\]
Because the likelihood table is invariant under simultaneously swapping
\(a\leftrightarrow b\) and \(z_1\leftrightarrow z_2\), any fixed point of \eqref{eq:self_consistency_alpha} would be symmetric, so we can just assume that 
\[
Q(a)=Q(b)=s,
\qquad
Q(g)=1-2s,
\qquad
0<s<\frac12.
\]
For this symmetric posterior candidate, the two normalising denominators in
the responsibility weights are equal
\[
A_\alpha(Q)
:=
\sum_{\theta\in\Theta} Q(\theta)p_\theta(z_1)^\alpha
=
\sum_{\theta\in\Theta} Q(\theta)p_\theta(z_2)^\alpha
=
q_g m^\alpha+s(h^\alpha+\varepsilon^\alpha).
\]
The self-consistency equation gives
\[
Q_\alpha^\star(\theta)
\propto
\exp\left\{
\frac{1}{\lambda\alpha}
\sum_{i=1}^2
\frac{p_\theta(z_i)^\alpha}{A_\alpha(Q_\alpha^\star)}
\right\}.
\]
Therefore the ratio of the posterior mass assigned to a specialist and to the
generalist is
\[
\frac{Q_\alpha^\star(a)}{Q_\alpha^\star(g)}
=
\exp\left\{
\frac{1}{\lambda\alpha A_\alpha(Q_\alpha^\star)}
\left[
h^\alpha+\varepsilon^\alpha-2m^\alpha
\right]
\right\}.
\]
The same expression holds for \(b\). Hence
\[
Q_\alpha^\star(a)>Q_\alpha^\star(g)
\qquad\Longleftrightarrow\qquad
h^\alpha+\varepsilon^\alpha>2m^\alpha.
\]
Thus, at the posterior level, the \(\alpha>0\) fixed point assigns more mass to
each specialist than to the generalist exactly when the \(\alpha\)-power mean of
the specialist likelihoods exceeds the generalist likelihood.  In the $\alpha\rightarrow 0$ limit, the objective reduces
to the usual Gibbs posterior,
\[
Q_0^\star(\theta)
\propto
\pi_0(\theta)
\exp\left\{
\frac{1}{\lambda}\sum_{i=1}^2\log p_\theta(z_i)
\right\}.
\]
With a uniform prior, the specialist-to-generalist posterior/likelihood ratio is therefore
\[
\frac{Q_0^\star(a)}{Q_0^\star(g)}
=
\left(
\frac{h\varepsilon}{m^2}
\right)^{1/\lambda}.
\]

Consequently, if $m^2>h\varepsilon$, then the classical Gibbs posterior favours the generalist.  On the other hand, for $\alpha > 0$, the self-consistent
posterior favours the specialists when
\[
m^2>h\varepsilon
\qquad\text{and}\qquad
h^\alpha+\varepsilon^\alpha>2m^\alpha.
\]
For example, with
\[
h=0.9,\qquad \varepsilon=0.01,\qquad m=0.3,
\]
we have \(m^2=0.09>0.009=h\varepsilon\),
so the \(\alpha=0\) Gibbs posterior favours the generalist. At the other extreme when
\(\alpha=1\),
\[
h+\varepsilon=0.91>0.6=2m,
\]
so the self-consistent \(\alpha=1\) posterior favours the specialists. The
transition occurs when
\[
h^\alpha+\varepsilon^\alpha=2m^\alpha,
\]
which for these values gives \(\alpha_{\mathrm{critical}}\approx 0.56\).
\end{example}

Example \ref{ex:posteriors} shows that the nonlinearity of the \(\alpha>0\) objective changes the posterior geometry itself, assigning mass according to self-consistent responsibility scores.  In this setting, a parameter can receive high posterior
mass because it explains a subset of the observations very well relative to the current
population \(Q_\alpha^\star\).

Another implication of the non-linearity in Eq. \eqref{eq:self_consistency_alpha} for $\alpha > 0$ is that, unlike for generalised Bayesian inference, we cannot rely on sequential or single-state sampling to compute the posterior distribution. In the generalised Bayesian settings ($\alpha \to 0$), the unnormalised posterior density at a single state $\theta$ depends only on the prior and the parameter's individual likelihood, which naturally permits the use of standard MCMC methods. In contrast, the responsibilities within the $\alpha$-R\'enyi posterior couple the state space, meaning parameter configurations cannot be evaluated in isolation. Consequently, the entire distribution must be evolved simultaneously to dynamically reconcile the interdependent predictive tilts. This requirement establishes a natural connection to other particle variational inference methods based on interacting particle systems, such as Stein Variational Gradient Descent (SVGD) \cite{svgd_2016}, where an ensemble of particles is evolved in parallel. However, whereas SVGD typically enforces distributional spread through an explicit repulsive kernel in the gradient updates, the $\alpha$-Rényi framework induces interaction and structural diversity directly through the non-linear routing within the objective itself.

While the $\alpha$-R\'enyi posterior differs in structure compared to generalised Bayesian posteriors over finitely many observations, we can show that they are consistent with respect to each other at the population level.   In the idealised setting where the true data-generating distribution lies strictly within our model class (i.e. well-specified case), the $\alpha$-R\'enyi population loss is still minimised by the true parameter, so that the ensemble is not forced to diversify even when a single, perfect predictor exists.  This is formalised in the following proposition, which builds on \cite[Lemma 2]{masegosa2020learning}.

\begin{proposition}[Well-specified population minimisers]
\label{prop:well_specified_minimizers}
Assume the model is well-specified, so that there exists \(\theta^\star\in\Theta\) with
$p_{\theta^\star}(\cdot\mid x)=p^\star(\cdot\mid x)
\quad \text{for }P_X^\star\text{-a.e. }x.
$
Then \(\delta_{\theta^\star}\) is a minimiser of the population risk $
\mathcal R_\alpha(Q):=\mathbb{E}_{P^\star}[\ell_\alpha(Q;X,Y)]$,
over \(Q \in \mathcal{P}(\Theta)\), for every \(\alpha\in(0,1]\).

For \(\alpha=1\), every minimiser satisfies \(p_Q=p^\star\), whereas for \(0<\alpha<1\), every minimiser  is supported on the exact-fit set
\[
\Theta^\star
:=
\left\{
\theta\in\Theta:
p_\theta(\cdot\mid x)
=
p^\star(\cdot\mid x)
\text{ for }P_X^\star\text{-a.e. }x
\right\}.
\]
\end{proposition}

\subsection{Stability under misspecification}\label{subsec:stability_misspecification}

To better understand how the $\alpha$-objective departs from the classical variational Bayes regime beyond the well-specified limit, we first study its local behaviour near $\alpha=0$. Let $L_\theta(x,y):=-\log p_\theta(y\mid x)$ denote the negative log-likelihood.  For sufficiently small $\alpha>0$, we can show that the lower bound in Lemma \ref{lem:alpha_jensen_gap_variance_bound} is tight, by taking a Taylor expansion of the cumulant generating function. This yields
\begin{equation}
\ell_\alpha(Q;x,y)
=
\ell_{0}(Q; x,y)
-
\frac{\alpha}{2}\operatorname{Var}_Q\!\big(L_\theta(x,y)\big)
+
\mathcal O(\alpha^2),
\label{eq:cumulant_expansion}
\end{equation}
where $\ell_0(Q; x, y) = \mathbb E_Q[L_\theta(x,y)]$ is the expected negative log-loss.  Eq.~\eqref{eq:cumulant_expansion} makes clear that positive \(\alpha\) introduces a first-order correction to the classical variational objective, which permits dispersion in the per-example loss under \(Q\).  Similar bounds have been derived in previous works, seeking to introduce second order corrections to Jensen's inequality, \cite{masegosa2020learning,liao2019sharpening,becker2012variance}.  Intuitively, this expansion shows that amongst distributions with comparable mean loss, positive $\alpha$ will favour ones whose ensemble members make different contributions to prediction.

The following result formalises this, showing that if $Q$ is close to a Dirac measure $\delta_{\theta}$, then inflating posterior variance will increase the population $\alpha$-risk, governed by a local stability operator.

\begin{proposition}[Local expansion around a Dirac posterior]
\label{prop:local_dirac_expansion}
Let \(z=(x,y)\), $P^\star$ be the data generating distribution, and assume \(L_\vartheta(z)=-\log p_\vartheta(y\mid x)\) is twice continuously differentiable in \(\vartheta\). Let \(Q\) be a probability measure on \(\Theta\) with mean \(\theta\) and covariance \(\Sigma\), with \(\Sigma\to 0\). Then
\begin{equation*}
\ell_\alpha(Q;z)
=
L_\theta(z)
+
\frac12 {\operatorname{Tr}}\big(\nabla^2 L_\theta(z)\Sigma\big)
-
\frac{\alpha}{2}\nabla \log p_\theta(z)^\top \Sigma \,\nabla \log p_\theta(z)
+
o(\|\Sigma\|).
\end{equation*}
Consequently, the population \(\alpha\)-risk admits the expansion
\begin{equation}
\mathcal R_\alpha(Q)
=
\mathbb{E}_{P^\star}[-\log p_\theta(Z)]
+
\frac12 {\operatorname{Tr}}\big((V(\theta)-\alpha J(\theta))\Sigma\big)
+
o(\|\Sigma\|),
\label{eq:population_local_dirac_expansion}
\end{equation}
where
\[
V(\theta)
:=
\mathbb{E}_{P^\star}\!\left[\nabla^2(-\log p_\theta(Z))\right],
\qquad
J(\theta)
:=
\mathbb{E}_{P^\star}\!\left[
\nabla \log p_\theta(Z)\nabla \log p_\theta(Z)^\top
\right].
\]
\end{proposition}

\begin{remark}[The \texorpdfstring{$V-\alpha J$}{V-alpha J} stability matrix and \texorpdfstring{$\alpha_{\text{critical}}$}{alpha critical}]\label{rem:stability_matrix}
Equation~\eqref{eq:population_local_dirac_expansion} shows that
\(V(\theta)-\alpha J(\theta)\) governs the local stability of a concentrated
population minimiser against infinitesimal spread in \(Q\)-space. In the well-specified case, Bartlett's identity implies $V(\theta^\star)=J(\theta^\star)=I(\theta^\star)$, resulting in a first-order variance penalty  proportional to $(1-\alpha)I(\theta^\star)$.
Thus, for \(\alpha<1\), Dirac posteriors remain locally stable to first order, while at \(\alpha=1\) this first-order stability vanishes. More generally, under misspecification, if \(V(\theta^\star)-\alpha J(\theta^\star)\) ceases to be positive definite, then small non-degenerate perturbations of the Dirac posterior may reduce the population risk, indicating that the objective can locally prefer spread over concentration at a single parameter value.

We can explicitly characterise the critical threshold $\alpha_{\text{\emph{critical}}}$ at which posterior spread becomes locally favourable under general misspecification. A concentrated posterior at $\theta^\star$ becomes unstable to spread in a direction $u \in \mathbb{R}^d$ if $u^\top (V(\theta^\star) - \alpha J(\theta^\star)) u < 0$. Assuming the Fisher information matrix $J(\theta^\star)$ is positive definite, the threshold for this phase transition is determined by the minimum of the generalised Rayleigh quotient
\begin{equation}
    \alpha_{\text{\emph{critical}}} = \min_{u \neq 0} \frac{u^\top V(\theta^\star) u}{u^\top J(\theta^\star) u}.
    \label{eq:alpha_critical}
\end{equation}
\end{remark}

\begin{corollary}[Local expansion with prior-potential regularisation]
\label{cor:local_dirac_with_prior}
Under the conditions of Proposition~\ref{prop:local_dirac_expansion}, suppose the prior \(\pi_0\) has twice continuously differentiable log-density in a neighbourhood of \(\theta\). Letting $U(\theta):=-\log \pi_0(\theta)$, we define the prior-regularised population functional $\mathcal F_\alpha(Q)
:=
\mathcal R_\alpha(Q)+\lambda\,\mathbb E_Q[U(\theta)]$, where $\mathbb E_Q[U(\theta)]$ is the local approximation to $\operatorname{KL}(Q || \pi_0)$ for sharply concentrated absolutely continuous posteriors.
Then, for a probability measure \(Q\) with mean \(\theta\) and covariance \(\Sigma\to 0\),
\[
\mathcal F_\alpha(Q)
=
\mathbb{E}_{P^\star}\!\left[-\log p_\theta(Z)\right]+\lambda U(\theta)
+
\frac12 \operatorname{Tr}\!\Big(\big(V(\theta)-\alpha J(\theta)+\lambda \nabla^2 U(\theta)\big)\Sigma\Big)
+
o(\|\Sigma\|).
\]
\end{corollary}
Thus in the presence of smooth prior regularisation, the local stability of a concentrated posterior is governed by the matrix $V(\theta)-\alpha J(\theta)+\lambda \nabla^2 U(\theta)$.  The prior curvature therefore acts as an additional force favouring concentration.
In the well-specified case, where \(V(\theta^\star)=J(\theta^\star)=I(\theta^\star)\), this becomes $(1-\alpha)I(\theta^\star)+\lambda \nabla^2 U(\theta^\star)$.

While the previous two results are stated at the population level to demonstrate the underlying statistical mechanism, the same local expansion applies to the empirical risk. This is achieved by replacing the population matrices $V$ and $J$ with their sample analogues
\[
\widehat V_N(\theta)
=
\frac1N\sum_{i=1}^N \nabla^2 L_\theta(z_i),
\qquad
\widehat J_N(\theta)
=
\frac1N\sum_{i=1}^N
\nabla \log p_\theta(z_i)\nabla \log p_\theta(z_i)^\top,
\]
so that the finite-sample objective has its own analogous local stability matrix
\[
\widehat V_N(\theta)-\alpha \widehat J_N(\theta),
\]
and an empirical critical value
\[
\widehat\alpha_{\mathrm{critical}}
=
\min_{u\neq0}
\frac{u^\top \widehat V_N(\theta)u}
{u^\top \widehat J_N(\theta)u}.
\]
In high-dimensional adapter spaces, this critical value can be efficiently estimated in a low-dimensional subspace using Hessian-vector and empirical Fisher-vector products.

To analyse the effect of $\alpha$ on local stability, we examine a simple example that illustrates the mechanism described in Proposition \ref{prop:local_dirac_expansion}, where everything can be computed closed form. In particular, we explicitly show how a positive value of \(\alpha\) can make a non-degenerate posterior preferable under contamination, and how this generates inflated predictive uncertainty in regions where the contamination is present.

\begin{example}[A two-regime linear-Gaussian example]
\label{example:stability}
Let \(X\in[-1,1]\), and consider a linear-Gaussian conditional model 
\begin{equation}
Y \mid X=x,\theta \sim \mathcal N\!\big(\phi(x)^\top \theta,\sigma^2\big),
\qquad
\phi(x):=\begin{pmatrix}x \\ x_+\end{pmatrix},
\qquad
x_+:=\max(x,0),
\label{eq:example_two_regime_model}
\end{equation}
with parameter \(\theta=(\theta_1,\theta_2)^\top \in \mathbb R^2\). The first feature \(x\) captures a global slope, while the second feature \(x_+\) allows the model to alter its behaviour only on the positive half-line.
We assume a Gaussian posterior ansatz
\[
Q=\mathcal N(m,\Sigma),
\qquad
m\in\mathbb R^2,\;\Sigma\in\mathbb R^{2\times 2},\;\Sigma\succeq 0.
\]
For this model,
\[
p_\theta(y\mid x)^\alpha
=
(2\pi\sigma^2)^{-\alpha/2}
\exp\!\left(
-\frac{\alpha}{2\sigma^2}(y-\phi(x)^\top\theta)^2
\right),
\]
and integrating over \(\theta\sim Q\) yields
\[
\int p_\theta(y\mid x)^\alpha\,Q(d\theta)
=
(2\pi\sigma^2)^{-\alpha/2}
\left(1+\frac{\alpha\,v_Q(x)}{\sigma^2}\right)^{-1/2}
\exp\!\left(
-\frac{\alpha (y-\phi(x)^\top m)^2}{2(\sigma^2+\alpha v_Q(x))}
\right),
\]
where
\[
v_Q(x):=\phi(x)^\top \Sigma \phi(x).
\]
Hence the pointwise \(\alpha\)-loss is
\begin{equation}
\ell_\alpha(Q;x,y)
=
\frac12 \log(2\pi\sigma^2)
+
\frac{1}{2\alpha}\log\!\left(1+\frac{\alpha\,v_Q(x)}{\sigma^2}\right)
+
\frac{(y-\phi(x)^\top m)^2}{2(\sigma^2+\alpha v_Q(x))}.
\label{eq:two_regime_alpha_loss}
\end{equation}

Suppose \(X\sim \mathrm{Unif}[-1,1]\), we define the true data-generating distribution $P^\star$ by 
\begin{equation*}
Y \mid X=x \sim
\begin{cases}
\mathcal N(\beta x,\sigma^2), & x<0, \\[4pt]
(1-\varepsilon)\,\mathcal N(\beta x,\sigma^2)
+\varepsilon\,\mathcal N((\beta+a)x,\sigma^2), & x\geq 0,
\end{cases}
\label{eq:two_regime_truthh}
\end{equation*}
with \(\varepsilon\in(0,1)\) and \(a\neq 0\). Thus the positive half-line contains a contaminated component, while the negative half-line remains clean. The conditional mean under \(P^\star\) is
\[
\mathbb E[Y\mid X=x]
=
\beta x + \varepsilon a\,x_+,
\]
so the pseudo-true conditional mean is represented exactly by
\[
m^\star
=
\begin{pmatrix}
\beta\\ \varepsilon a
\end{pmatrix}.
\]
However, the conditional variance under \(P^\star\) is
\[
\mbox{Var}(Y\mid X=x)
=
\sigma^2+\varepsilon(1-\varepsilon)a^2 x_+^2,
\]
which is heteroscedastic and therefore cannot be represented by the homoscedastic model \eqref{eq:example_two_regime_model}. 

Since the Gaussian negative log-likelihood is quadratic, the local stability objects can be computed explicitly at \(m^\star\). First,
\[
L_\theta(x,y)
=
\frac{(y-\phi(x)^\top\theta)^2}{2\sigma^2}
+\frac12\log(2\pi\sigma^2),
\]
and its Hessian
\[
\nabla_\theta^2 L_\theta(x,y)
=
\frac{1}{\sigma^2}\phi(x)\phi(x)^\top.
\]
Hence,
\begin{equation*}
V(m^\star)
=
\frac{1}{\sigma^2}\,\mathbb{E}\!\left[\phi(X)\phi(X)^\top\right].
\label{eq:two_regime_V}
\end{equation*}

On the other hand, the score is given by
\[
\nabla_\theta \log p_\theta(y\mid x)
=
\frac{y-\phi(x)^\top\theta}{\sigma^2}\,\phi(x),
\]
so the Fisher information results in 
\begin{equation*}
J(m^\star)
=
\frac{1}{\sigma^4}\,\mathbb{E}\!\left[
\mbox{Var}(Y\mid X)\,\phi(X)\phi(X)^\top
\right]
=
\frac{1}{\sigma^2}\mathbb{E}[\phi(X)\phi(X)^\top]
+
\frac{\varepsilon(1-\varepsilon)a^2}{\sigma^4}
\mathbb{E}\!\left[X_+^2\,\phi(X)\phi(X)^\top\right].
\label{eq:two_regime_J}
\end{equation*}
Therefore the local spread-direction stability matrix is
\begin{equation}
M_\alpha
:=
V(m^\star)-\alpha J(m^\star)
=
\frac{1-\alpha}{\sigma^2}\mathbb{E}[\phi(X)\phi(X)^\top]
-
\frac{\alpha\,\varepsilon(1-\varepsilon)a^2}{\sigma^4}
\mathbb{E}\!\left[X_+^2\,\phi(X)\phi(X)^\top\right].
\label{eq:example_two_regime_Malpha}
\end{equation}

Proposition~\ref{prop:local_dirac_expansion}, for a small covariance perturbation \(\Sigma\) around the Dirac posterior at \(m^\star\),
\[
\mathcal R_\alpha(Q)
=
R(m^\star)
+
\frac12 \operatorname{Tr}(M_\alpha \Sigma)
+
o(\|\Sigma\|).
\]
Thus the most favourable direction of posterior spread is given by the eigenvector of \(M_\alpha\) associated with its smallest eigenvalue. If all eigenvalues are positive, the Dirac posterior is locally stable. If the smallest eigenvalue becomes negative, then spreading in the corresponding eigendirection lowers the population \(\alpha\)-risk.

Equation~\eqref{eq:example_two_regime_Malpha} shows why contamination on the positive half-line creates a selective instability. The contamination-specific correction is proportional to
\[
\mathbb{E}\!\left[X_+^2\,\phi(X)\phi(X)^\top\right],
\]
which acts only where \(x\ge 0\). Consequently, the smallest-eigenvalue eigendirection of \(M_\alpha\) is biased toward the regime-specific coordinate associated with \(x_+\), so that the contamination perturbs the posterior geometry anisotropically,  encouraging spread along the parameter direction that controls predictions on the contaminated region.

To find the most favourable direction of posterior spread, we restrict attention to covariance perturbations of the form $\Sigma = s^2 uu^\top$, where $u$ is a unit vector. By Remark \ref{rem:stability_matrix}, the exact threshold for instability along any specific direction $u$ is given by the generalised Rayleigh quotient \eqref{eq:alpha_critical}.
Solving the eigenvalue problem $\det(V(m^\star) - \alpha J(m^\star)) = 0$ reveals that the principal eigenvector (associated to $\alpha_{\text{critical}}$) is precisely the contamination-specific direction $u = e_2 = \begin{pmatrix} 0 & 1 \end{pmatrix}^\top$. To compute the absolute critical threshold, we evaluate the Rayleigh quotient for this optimal direction. The numerator and denominator are given by 
\begin{align*}
    u^\top V(m^\star) u = \frac{1}{\sigma^2}\mathbb{E}[X_+^2], \quad 
    u^\top J(m^\star) u = \frac{1}{\sigma^2}\mathbb{E}[X_+^2] + \frac{\varepsilon(1-\varepsilon)a^2}{\sigma^4}\mathbb{E}[X_+^4].
\end{align*}

Substituting these quadratic forms into \eqref{eq:alpha_critical} yields 
\begin{equation*}
\alpha_{\mathrm{critical}} = \frac{\sigma^2\,\mathbb{E}[X_+^2]}{\sigma^2\, \mathbb{E}[X_+^2] + \varepsilon(1-\varepsilon)a^2\,\mathbb{E}[X_+^4]}.
\label{eq:two_regime_alpha_critt}
\end{equation*}
For \(X\sim \mathrm{Unif}[-1,1]\), one has $\mathbb E[X_+^2]=\frac16$ and $\mathbb E[X_+^4]=\frac1{10}$, and therefore
\begin{equation*}
\alpha_{\mathrm{critical}}
=
\frac{5\sigma^2}{5\sigma^2+3\varepsilon(1-\varepsilon)a^2}.
\label{eq:two_regime_alpha_crit_uniform}
\end{equation*}

\begin{figure}[htb]
    \centering
    \includegraphics[width=0.8\linewidth]{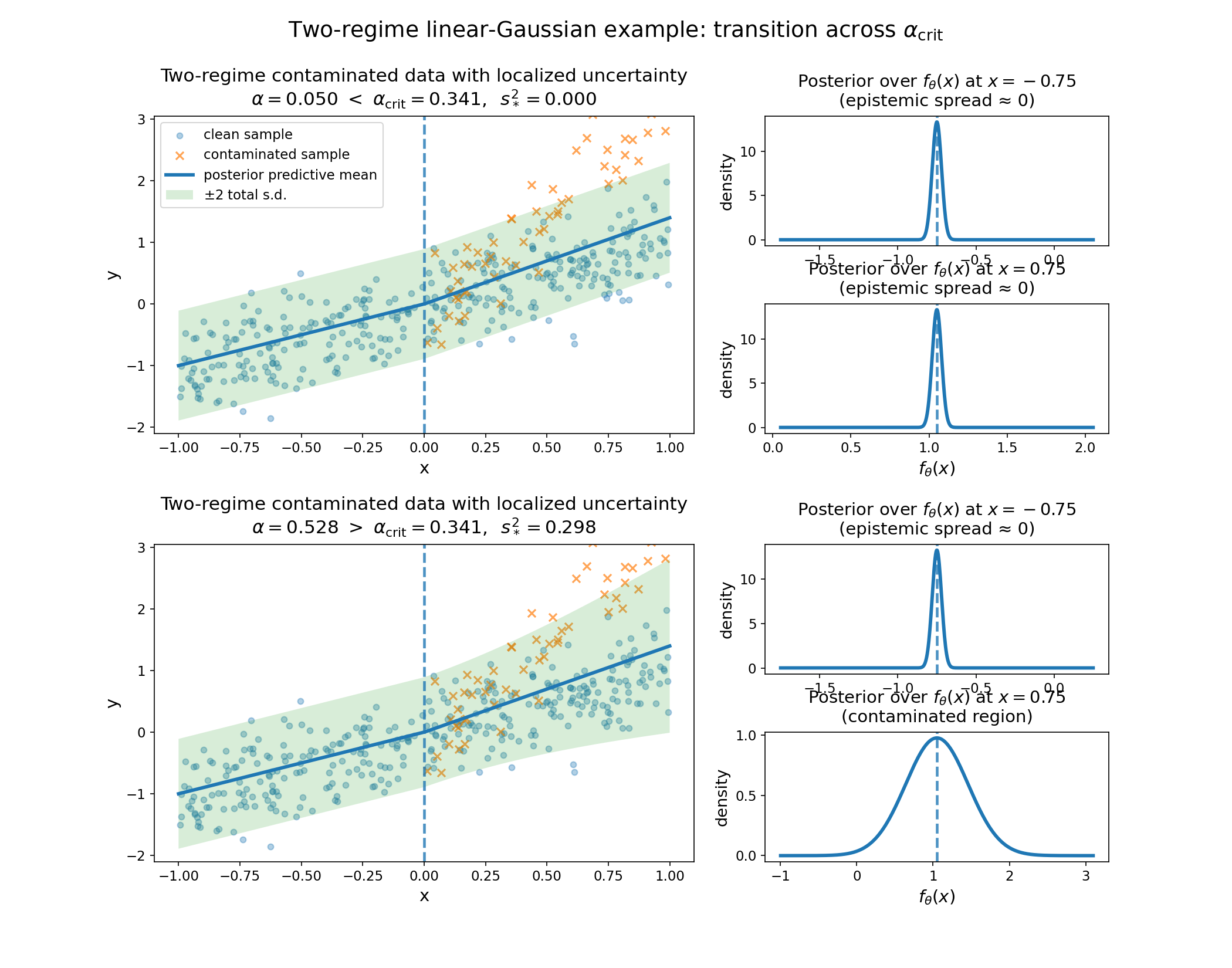}
    \caption{Behaviour of the model in Example \ref{example:stability} for $\alpha$ above and below the critical threshold.
    (Left) Clean and contaminated samples. (Right) Posterior predictive distributions when $x \geq 0$ and $x < 0$. }
    \label{fig:two_regime_example}
\end{figure}

The posterior predictive mean and variance under \(Q=\mathcal N(m^\star,\Sigma)\) are
\[
\mathbb E_Q[Y\mid X=x]
=
\phi(x)^\top m^\star
=
\beta x+\varepsilon a\,x_+,
\]
and
\[
\mbox{Var}_Q(Y\mid X=x)
=
\underbrace{\sigma^2}_{\text{aleatoric}}
+
\underbrace{\phi(x)^\top \Sigma\, \phi(x)}_{\text{epistemic}}.
\]
Restricting the posterior covariance to the contamination-specific subspace, $\Sigma=\mathrm{diag}(0,s^2)$, this becomes
\begin{equation*}
\mbox{Var}_Q(Y\mid X=x)=\sigma^2+s^2x_+^2,
\qquad
\mbox{Var}_{\theta\sim Q}(\mathbb E[Y\mid X=x,\theta])=s^2x_+^2.
\label{eq:two_regime_epistemic}
\end{equation*}
Thus the epistemic component vanishes for \(x<0\) and grows only on the contaminated region \(x\ge 0\).

We draw two conclusions from this simple example.  First, the eigenstructure of \(M_\alpha=V-\alpha J\) identifies the posterior spread directions that become locally favorable as \(\alpha\) increases. Second, when we align posterior spread with the contamination-specific feature \(x_+\), the resulting posterior spread produces {predictive} epistemic uncertainty, concentrated precisely on the anomalous region of the input space rather than inflated uniformly across all inputs. See Figure \ref{fig:two_regime_example} for an illustration of posterior uncertainties in each regime.
\end{example}

\subsection{Responsibility-weighted influence and robustness to contamination}\label{sec:robustness_influence}

The \(\alpha\)-R\'enyi variational objective also induces a form of responsibility-weighted
robustness. This should not be interpreted as unconditional robustness
of a single model in the sense of \cite{Huber.Wiley.ea1981Robuststatistics}. Rather, the robustness mechanism depends on the ensemble itself, where
an anomalous or poisoned observation can be routed away from particles with
which it is incompatible.

Let \(0<\alpha<1\), and consider a finite particle approximation  $Q_M = \frac{1}{M}\sum_{i=1}^M \delta_{\theta_i}$.  Then the gradient of $\ell_{\alpha}(Q_M; x, y)$ with respect to $\theta_i$ is given by
$$
\nabla_{\theta_i}\ell_{\alpha}(Q_M; x, y) = -\frac{p_{\theta_i}(y \mid x)^{\alpha}}{\frac{1}{M}\sum_{j} p_{\theta_j}(y \mid x)^\alpha } \nabla_{\theta_i} \log p_{\theta_i}(y \mid x) = w^{(\alpha)}(\theta_i; x,y, Q)\nabla_{\theta_i} L_{\theta_i}(x,y).
$$

We see that each gradient is modulated by the $\alpha$-responsibility weight $w^{(\alpha)}$.  Thus, the influence of an observation on particle \(i\) is its ordinary
log-likelihood score multiplied by an observation-specific responsibility.  This responsibility admits a simple shielding bound. Let
\[
L_{\min}(x,y):=\min_{1\le j\le M}L_{\theta_j}(x,y).
\]
Then
\[
w^{(\alpha)}(\theta_i; x,y, Q_M)
=
\frac{e^{-\alpha L_{\theta_i}(x,y)}}{\sum_j e^{-\alpha L_{\theta_j}(x,y)}}
\le
\exp\big(-\alpha[L_{\theta_i}(x,y)-L_{\min}(x,y)]\big).
\]
Consequently,
\begin{equation}
\|\nabla_{\theta_i}\ell_\alpha(Q^M;x,y)\|
\le
\exp\big(-\alpha[L_{\theta_i}(x,y)-L_{\min}(x,y)]\big)
\,
\|\nabla_{\theta_i}L_{\theta_i}(x,y)\|.
\label{eq:responsibility_shielding_bound}
\end{equation}
Equation~\eqref{eq:responsibility_shielding_bound} shows that the influence of
\(z\) on particle \(i\) is exponentially attenuated in the loss gap between
particle \(i\) and the best particle for that observation.

This provides a concrete mechanism for robustness to contamination or data poisoning. If a poisoned example \((x,y)\) is highly
incompatible with a clean particle \(i\), but is better explained by another
particle \(j\), then \(L_{\theta_i}(x,y)-L_{\theta_j}(x,y)\) is large, and the poisoned example sends
exponentially little gradient to the clean particle. Instead, its gradient is
routed toward the part of the ensemble that can explain it. In this sense,
positive \(\alpha\) can quarantine inconsistent or poisoned examples into a
small subset of particles, rather than forcing every particle to absorb the same
corrupted update.

This robustness mechanism is very diferent from ordinary robust losses arising from density-power divergences, e.g. \cite[Eq. (2.7)]{robustness_power_divergence_basu_98}.  The redescending influence of the $\alpha$-R\'enyi variational objective strongly depends on the size and diversity of $Q$. Robustness arises only when \(Q\) retains enough diversity
for an anomalous observation to be relatively incompatible with some particles
and relatively compatible with others.  Crucially, if all particles assign similar low likelihood
to the same outlier, then \(L_{\theta_i}(x,y)-L_{\min}(x,y)\) is small for many particles and
the normalised responsibilities need not suppress the gradient, so that the
mechanism is relative rather than absolute.

This suggests two practical diagnostics for poisoning or heterogeneous
contamination. First, the effective sample size of the responsibilities,
\[
\operatorname{ESS}(z)
=
\left(\sum_{i=1}^M w^{(\alpha)}(\theta_i; x,y, Q_M)^2\right)^{-1},
\]
measures how many particles are absorbing a given observation. Poisoned or
conflicting examples should often have low ESS, indicating concentrated routing.
Second, examples for which all particles have high loss but no clear
responsibility concentration indicate out-of-support anomalies; these are not
automatically handled by the ensemble and may require abstention, filtering, or
additional model support.

Ensemble and aggregation approaches to mitigating the effects of poisoning have been studied in previous works, \cite{levinedeep, jia2021intrinsic, wang2022improved, pmlr-v202-rezaei23a}.   Many of these strategies rely on partitioning the data, and training an ensemble of models across the partition, aggregating predictions at inference time.   Thus, they protect predictions by limiting the number of base models that any poisoned example can affect.   While related, our mechanism is different: all particles are trained jointly on the same data, and poisoning robustness arises through responsibilities that route gradients away from
particles for which an example is incompatible.

\section{Finite-particle approximations}
\label{sec:finite_particle_training}

The preceding sections define the \(\alpha\)-R\'enyi objective as a variational problem over probability measures \(Q\in\mathcal P(\Theta)\). In practice, we do not directly optimise over arbitrary measures. Instead, we approximate \(Q\) using a finite ensemble of \(M\) particles, $Q^M=\frac1M\sum_{i=1}^M \delta_{\theta_i}$,
where each \(\theta_i\) parametrises the model state. This turns the variational problem into a finite-dimensional optimisation problem over the particle parameters \(\theta_1,\ldots,\theta_M\), which can be solved using mini-batch AdamW.
However, directly substituting the empirical measure $Q^M$ into the objective $\mathcal{F}_\alpha$ is ill-posed:
the $\operatorname{KL}$ term \(\operatorname{KL}(Q^M\|\pi_0)\) becomes singular for any continuous prior density $\pi_0$. We therefore consider two practical alternatives: a smoothed empirical-density approximation to the $\operatorname{KL}$, and a simpler prior-potential surrogate which drops the entropy term. The full algorithm is provided in Alg. \ref{alg:alpha_renyi_lora_sft}.

\subsection{Finite-particle supervised objective}\label{sec:sft_subsec}

For a supervised fine-tuning example \((x,y)\), where \(y=(y_1,\ldots,y_T)\) is the target response, particle \(i\) assigns the autoregressive sequence log-likelihood $s_i(x,y) := \log p_{\theta_i}(y\mid x)$.  The \(\alpha\)-R\'enyi finite-particle loss is
\begin{equation}
\ell_\alpha^{(M)}(\theta_{1:M};x,y)
=
-\frac1\alpha
\log\left(
\frac1M\sum_{i=1}^M \exp(\alpha s_i(x,y))
\right),
\qquad \alpha>0.
\label{eq:finite_particle_sft_loss}
\end{equation}
In the limit \(\alpha\to 0\), this becomes the average negative log-likelihood across particles
\begin{equation*}
\ell_0^{(M)}(\theta_{1:M};x,y)
=
-\frac1M\sum_{i=1}^M s_i(x,y).
\label{eq:finite_particle_sft_loss_alpha0}
\end{equation*}

For a minibatch \(\mathcal B=\{(x_b,y_b)\}_{b=1}^B\), the empirical data loss is
\begin{equation}
\widehat{\mathcal L}_{\alpha}^{\mathrm{SFT}}(\theta_{1:M})
=
\frac{N}{B}\sum_{b=1}^B
\ell_\alpha^{(M)}(\theta_{1:M};x_b,y_b),
\label{eq:minibatch_sft_loss}
\end{equation}
where $N$ is the total size of the training set.
The corresponding responsibilities are
\begin{equation}
w_{i,b}^{(\alpha)}
=
\frac{\exp(\alpha s_i(x_b,y_b))}
{\sum_{j=1}^M \exp(\alpha s_j(x_b,y_b))}
=
\frac{p_{\theta_i}(y_b\mid x_b)^\alpha}
{\sum_{j=1}^M p_{\theta_j}(y_b\mid x_b)^\alpha}.
\label{eq:sft_responsibilities_finite}
\end{equation}
Differentiating \eqref{eq:minibatch_sft_loss} gives
\begin{equation}
\nabla_{\theta_i}\widehat{\mathcal L}_{\alpha}^{\mathrm{SFT}}
=
-\frac{N}{B}\sum_{b=1}^B
w_{i,b}^{(\alpha)}
\nabla_{\theta_i}\log p_{\theta_i}(y_b\mid x_b).
\label{eq:sft_particle_gradient}
\end{equation}
Thus \(\alpha\) determines how examples are routed across particles. For \(\alpha\to 0\), all particles receive equal responsibility. For \(\alpha>0\), examples contribute more strongly to particles that already assign them higher likelihood.

In practice, \eqref{eq:finite_particle_sft_loss} is implemented via a numerically stable log-sum-exp
\begin{equation*}
\ell_\alpha^{(M)}(\theta_{1:M};x,y)
=
-\frac1\alpha
\left[
\operatorname{logsumexp}_{i=1}^M(\alpha s_i(x,y))
-\log M
\right].
\label{eq:stable_sft_loss}
\end{equation*}
Turning to the prior regularisation term $\lambda \operatorname{KL}(Q\|\pi_0)$, for a density \(q=dQ/d\theta\) and prior density \(\pi_0(\theta)\), this can be decomposed as
\begin{equation*}
\operatorname{KL}(Q\|\pi_0)
=
\int q(\theta)\log q(\theta)\,d\theta
-
\int q(\theta)\log \pi_0(\theta)\,d\theta.
\label{eq:kl_decomposition}
\end{equation*}
The first term is the negative differential entropy of \(Q\), while the second is the prior-potential term. For the empirical particle measure $Q^M$
the entropy term is singular and \(\operatorname{KL}(Q^M\|\pi_0)=\infty\) whenever \(\pi_0\) is absolutely continuous. Therefore, an implementable finite-particle objective must approximate or modify the $\operatorname{KL}$ term.
We consider two practical choices.

\paragraph{Option A: smoothed empirical KL.}
We can retain an approximation to the full KL divergence by replacing the discrete empirical measure $Q^M$ with a continuous kernel density estimate
\[
q_M^\varepsilon(\theta)
=
\frac1M\sum_{i=1}^M K_\varepsilon(\theta-\theta_i),
\]
where \(K_\varepsilon\) is a smooth, positive kernel of bandwidth \(\varepsilon>0\). The smoothed KL regulariser is then
\begin{equation}
\mathcal R_{\mathrm{KDE}}(\theta_{1:M})
=
\lambda \int q_M^\varepsilon(\theta)
\log\frac{q_M^\varepsilon(\theta)}{\pi_0(\theta)}\,d\theta.
\label{eq:kde_kl_regularizer}
\end{equation}
This term preserves both parts of the KL: it encourages particles to remain in regions of high prior density, while the entropy component prevents the ensemble from collapsing to a single point.

In practice, the integral in \eqref{eq:kde_kl_regularizer} is generally intractable in high dimensions. A common approximation is to evaluate the smoothed density at the particle locations, yielding
\begin{equation*}
\widehat{\mathcal R}_{\mathrm{KDE}}(\theta_{1:M})
=
\frac{\lambda}{M}\sum_{i=1}^M
\left[
\log q_M^\varepsilon(\theta_i)
-
\log \pi_0(\theta_i)
\right],
\label{eq:kde_particle_regularizerr}
\end{equation*}
where $q_M^\varepsilon(\theta_i)
=
\frac1M\sum_{j=1}^M K_\varepsilon(\theta_i-\theta_j)$.
The gradient of the first term induces a repulsive interaction between nearby particles, while the second term pulls particles toward regions of high prior density. 
Although this approach successfully preserves both components of the exact KL-regularised variational problem, it introduces a bandwidth hyperparameter $\varepsilon$ that is difficult to tune in high-dimensional parameter spaces.

\paragraph{Option B: prior-potential surrogate}

This approach drops the entropy term entirely, retaining only the prior-potential component
\begin{equation*}
\mathcal R_{\mathrm{prior}}(\theta_{1:M})
=
-\frac{\lambda}{M}\sum_{i=1}^M \log \pi_0(\theta_i).
\label{eq:prior_potential_regularizer}
\end{equation*}
For a Gaussian prior \(\pi_0=\mathcal N(0,\tau^2 I)\), this reduces to
\begin{equation*}
\mathcal R_{\mathrm{prior}}(\theta_{1:M})
=
\frac{\lambda}{2M\tau^2}
\sum_{i=1}^M \|\theta_i\|^2
+
\mathrm{const.}
\label{eq:gaussian_prior_regularizer}
\end{equation*}
This approximation is simple and can be viewed as an analogue of standard weight decay regularisation in neural networks.   However, by dropping the entropy term, there is no explicit protection against particle collapse.   Ensemble diversity relies entirely on initialisation, and to a lesser extent stochasticity during the optimisation process.   Careful initialisation is therefore critical to ensure the ensemble maintains its spread during training.

\paragraph{Final objective}
The practical supervised objective is
\begin{equation*}
\mathcal J_{\alpha}^{\mathrm{SFT}}(\theta_{1:M})
=
\widehat{\mathcal L}_{\alpha}^{\mathrm{SFT}}(\theta_{1:M})
+
\mathcal R(\theta_{1:M}),
\label{eq:direct_final_sft_objectivee}
\end{equation*}
where \(\mathcal R\) is either \(\mathcal R_{\mathrm{prior}}\) or \(\widehat{\mathcal R}_{\mathrm{KDE}}\). 

\subsection{Gradient-flow interpretation}
\label{subsec:gradient_flow_interpretation}

The finite-particle objective above can be viewed as a scalable approximation to
a measure-valued gradient flow. Although this formalism is not needed for
implementation, it does shed light on how responsibility-weighted updates arise
naturally.

Consider the $\alpha$-R\'enyi data-fit term 
\[\mathcal{D}_\alpha(Q) := \sum_{i=1}^N \ell_{\alpha}(Q, x_i, y_i) = -\frac{1}{\alpha}\sum_{i=1}^N \log\int_{\Theta} p_{\theta}(x_i,y_i)^{\alpha}Q(d\theta).\]  
This has first variation 
\[
\frac{\delta \mathcal D_\alpha}{\delta Q}(\theta)
=
-\frac1\alpha
\sum_{i=1}^N
w^{(\alpha)}(\theta;x_i,y_i,Q).
\]

Assuming that \(Q_t\) admits a density \(\rho_t\) and the prior has density \(\pi_0\), the
formal Wasserstein gradient flow \citep{ollivier2014optimal} of the full $\alpha$-R\'enyi variational objective
\[
\mathcal F_\alpha(Q)
=
\mathcal D_\alpha(Q)
+
\lambda \operatorname{KL}(Q\|\pi_0)
\]
is
\[
\partial_t \rho_t
=
\nabla_\theta\cdot
\left(
\rho_t
\nabla_\theta \Psi_t
\right),
\qquad
\Psi_t(\theta)
:=
-\frac1\alpha
\sum_{i=1}^N
w^{(\alpha)}(\theta;x_i,y_i,\rho_t)
+
\lambda \log\frac{\rho_t(\theta)}{\pi_0(\theta)}.
\]
Equivalently, expanding the KL contribution gives the nonlinear
Fokker-Planck equation
\[
\partial_t \rho_t
=
-\nabla_\theta\cdot
\left[
\rho_t
\left(
\sum_{i=1}^N
w^{(\alpha)}(\theta;z_i,\rho_t)
\nabla_\theta \log p_\theta(z_i)
+
\lambda \nabla_\theta\log \pi_0(\theta)
\right)
\right]
+
\lambda \Delta \rho_t.
\]
The first term is the responsibility-weighted likelihood drift, the second
pulls mass toward high-prior-density regions, and the diffusion term is the
entropy component of the KL. Along sufficiently regular solutions,
\begin{equation}
\label{eq:lyapunov}
\frac{d}{dt}\mathcal F_\alpha(\rho_t)
=
-
\int
\|\nabla_\theta \Psi_t(\theta)\|^2
\rho_t(\theta)\,d\theta
\le 0.
\end{equation}
Thus, this flow decreases the same variational objective.  Despite having established convexity in total variation of the functional $\mathcal{F}_{\alpha}$ in Lemma \ref{lem:convexity_alpha_objective}, this is not sufficient to establish a qualitative rate of convergence for $\rho_t$ through \eqref{eq:lyapunov}.  This would require establishing displacement / geodesic convexity with respect to the Wasserstein geometry, e.g. by following the programme of \cite{carrillo2006contractions}.  While this strategy has been employed to obtain quantitative rates of convergence for other particle based variational inference methods, e.g. \cite{duncan2023geometry}, we do not expect the strategy to be applicable in the setting of this paper.  Thus stationary points of the Wasserstein gradient flow of $\mathcal{F}_{\alpha}$ need not be minimisers of the objective.   Nonetheless, it is possible that Wasserstein stationary points of $\mathcal{F}_{\alpha}$ would still inherit  useful properties, in the spirit of \cite{chen2025stationary}, but we defer this analysis for future work.

To approximate the gradient flow we plug in the empirical approximation to a general probability distribution $Q$,
\[
Q^M_t=\frac1M\sum_{i=1}^M \delta_{\theta_i(t)},
\]
and optimise the finite-particle objective. For the prior-potential surrogate,
the corresponding deterministic particle flow is, up to the normalisation
conventions in the empirical loss,
\[
\dot\theta_i
=
\sum_{b=1}^B
w_{i,b}^{(\alpha)}
\nabla_{\theta_i}\log p_{\theta_i}(z_b)
-
\nabla_{\theta_i}\mathcal R(\theta_{1:M}),
\qquad i=1,\ldots,M.
\]
This is precisely the negative Euclidean gradient flow of
\(\mathcal J_\alpha^{\mathrm{SFT}}\) over a finite particle ensemble.   In this paper, we do not adopt these dyanamics, as it does not scale well to large LLM post-training applications.  Instead of an explicit Euler discretisation of
this particle flow, we use AdamW which is highly effective for training large-scale neural network models. Heuristically, AdamW can be viewed understood as a practical
stochastic, adaptive, preconditioned discretisation of the finite-particle
energy \(\mathcal J_\alpha^{\mathrm{SFT}}\), rather than a scheme which is targeting a different objective.   

We believe that one could derive gradient flows with respect to a different preconditioned metric, which would yield AdamW-like numerical schemes for optimising \(\mathcal J_\alpha^{\mathrm{SFT}}\).   We leave this for future work.

\section{Experiments: \texorpdfstring{$\alpha$}{alpha}-R\'enyi ensembles for uncertainty quantification in LLMs}
\label{sec:llm_ensembles}

We now instantiate the general \(\alpha\)-R\'enyi framework in the setting of LLM post-training. Our goal is to replace the usual single adapted model with a finite ensemble of interacting low-rank adaptations applied to a shared, frozen base model. In this setting, the particles are trainable LoRA modules attached to the same underlying transformer. 

Let \(f_{W_0}\) denote a pretrained autoregressive language model with frozen parameters \(W_0\). We introduce \(M\) trainable low-rank adaptations $\theta_1,\dots,\theta_M \in \Theta$, where each \(\theta_i\) parametrises a collection of LoRA updates applied to a chosen subset of linear maps in the transformer. For a given particle \(\theta_i\), the effective model is $W(\theta_i)=W_0+\Delta W(\theta_i)$, and the induced conditional distribution over output sequences is denoted by $p_{\theta_i}(y\mid x)
=
p_{W_0+\Delta W(\theta_i)}(y\mid x)$. See Figure \ref{fig:lora_alpha_renyi_ensemble} for an illustration of the framework.  The full algorithm can be found in App. \ref{app:algorithm}, together with a code snippet. Further experimental results are provided in App. \ref{app:experimental_details}.

\subsection{\texorpdfstring{$\alpha$}{alpha}-R\'enyi ensemble training for supervised fine-tuning}
\label{subsec:sft_variant}

Following Section \ref{sec:sft_subsec}, we consider the total finite-particle objective ${\mathcal J}_{\alpha}^{\mathrm{SFT}} = \widehat{\mathcal L}_{\alpha}^{\mathrm{SFT}} + \frac{1}{N}\mathcal R_{\mathrm{prior}}$. For $\alpha>0$, this objective promotes specialisation across the LoRA ensemble.   
\paragraph{Experimental setup} 
We evaluate the framework using representative base models from different families, in particular, Phi-3-mini-4k-instruct, Qwen2-1.5B-Instruct and Nemotron-3-8B-base-4k. 
For each model, we instantiate an ensemble of $M=8$ trainable LoRA particles attached to the shared, frozen transformer backbone. 
We use the MMLU benchmark \cite{hendrycks2021measuring_mmlu} (split into training and test) as our primary dataset due to its diversity as it contains questions from a wide range of subjects. See App. \ref{app:experimental_details} for more details.

\paragraph{Results} Figure \ref{fig:heatmap_comparison} illustrates the emergent specialisation of the $\alpha$-Rényi ensemble on the diverse MMLU benchmark across the three base models. Under the classical variational objective ($\alpha=0$), the ensemble demonstrates uniform behaviour across the $M=8$ LoRA particles, yielding homogeneous performance that lacks differentiated feature learning. In contrast, setting the interpolation parameter to $\alpha=0.8$ induces distinct, localised specialisation among the adapters. 
By filtering the evaluation to isolate only the subset of queries that the frozen base model answers incorrectly, the third column reveals an interesting underlying mechanism. The induced specialisation is primarily driven by improvements on these incorrect examples, confirming that the $\alpha>0$ objective resolves base-model deficiencies by distributing challenging, heterogeneous tasks across complementary specialists.

 \begin{figure}[t]
    \centering
    \begin{subfigure}[c]{1\textwidth}
        \centering
        \includegraphics[width=\textwidth]{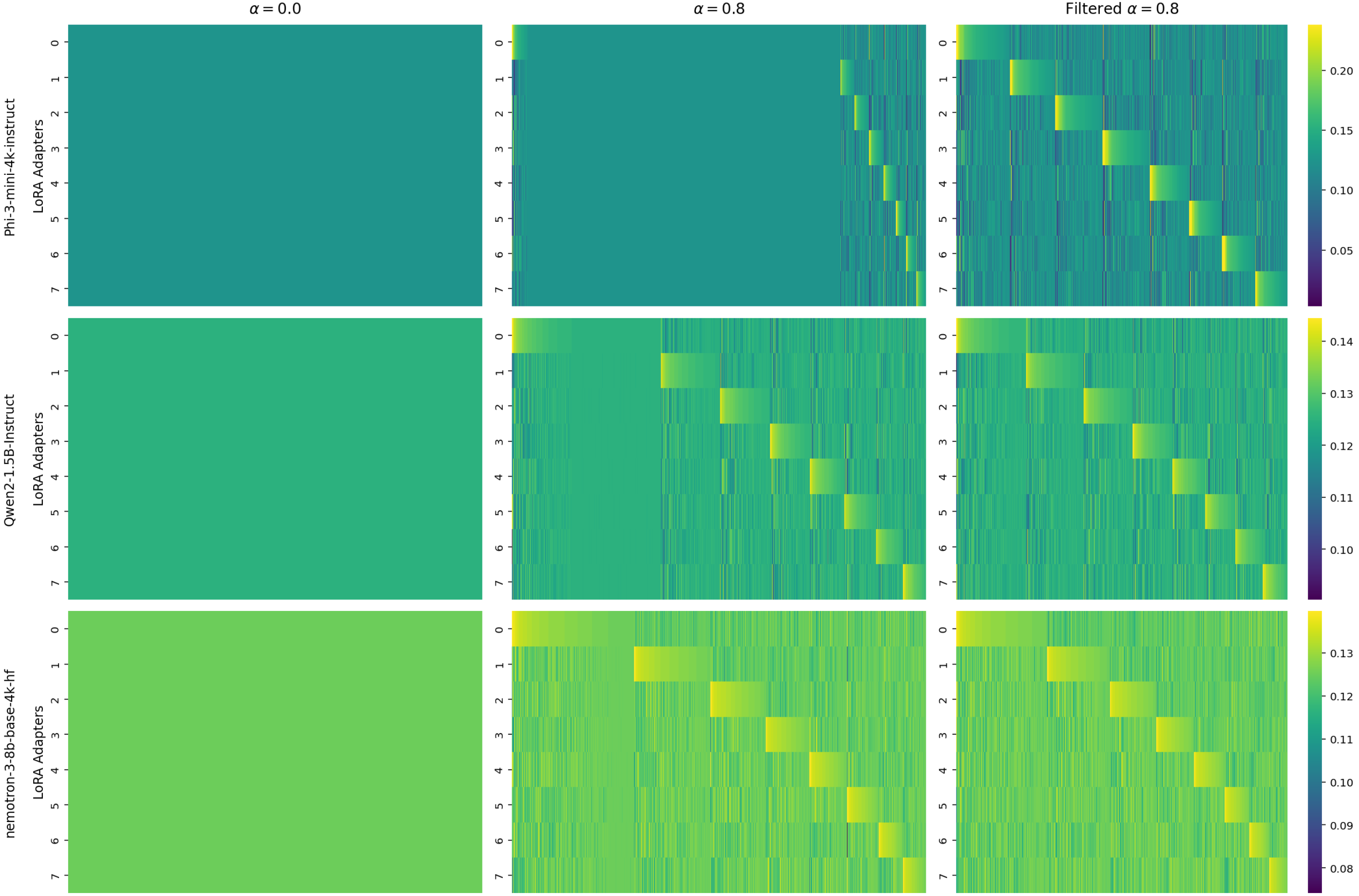}
        \label{fig:alpha_zero}
    \end{subfigure}
    \hfill 
    \caption{Emergent specialisation on the MMLU benchmark. Model performances (rows) are compared across $\alpha=0$ and $\alpha=0.8$ (columns). By filtering out questions the base model already answered correctly, the final column reveals that specialisation is primarily driven by improvements on incorrect examples.}
    \label{fig:heatmap_comparison}
\end{figure}

\subsection{\texorpdfstring{$\alpha$}{alpha}-R\'enyi ensemble training for direct preference optimisation}
\label{subsec:dpo_variant}

While the supervised fine-tuning objective is natural when a single target response is available, alignment-oriented settings often rely on pairwise preference data of the form $(x, y^+, y^-)$, where $y^+$ is preferred over $y^-$. We can extend the \(\alpha\)-R\'enyi framework to this regime by formulating the objective directly at the level of preference likelihoods rather than token-level imitation.

Following the Direct Preference Optimization (DPO) framework, each particle \(\theta_i\) induces its own Bradley-Terry preference likelihood $r_i(x,y^+,y^-)$ relative to a frozen reference model. We evaluate the ensemble through an \(\alpha\)-aggregated preference loss
\begin{equation}
\ell_\alpha^{\mathrm{pref}}(Q^M; x,y^+,y^-)
=
-\frac{1}{\alpha}
\log\left(
\frac1M\sum_{i=1}^M r_i(x,y^+,y^-)^\alpha
\right).
\label{eq:alpha_preference_loss}
\end{equation}

As in the supervised setting, this objective induces responsibilities across particles, routing preference pairs toward the particles that currently explain them well. Rather than forcing a single model to satisfy all pairwise preference constraints simultaneously, which can lead to brittle compromises, different preference pairs may be explained by different particles. This allows the ensemble to represent conflicting alignment pressures through posterior diversity and predictive cooperation.  Full derivations of the particle-wise preference margins, the minibatch objective, and the gradient updates are provided in App. \ref{app:dpo_variant}.

\paragraph{Experimental setup} We use the same language models as in Section \ref{subsec:sft_variant}. To evaluate the ensemble's behaviour under ambiguity, we use the OR benchmark \cite{cui2025or} after having fine-tuned on another DPO dataset to avoid leakage. The OR dataset serves as a robust testbed to determine whether epistemic uncertainty, manifested as disagreement across particles, effectively captures the model's confidence when deciding to refuse a prompt.

\paragraph{Results}
Table \ref{tab:dpo_results} demonstrates that the $\alpha$-Rényi ensemble successfully captures alignment ambiguity through inter-particle disagreement. When evaluated on the OR benchmark, the ensemble translates conflicting safety pressures into observable epistemic uncertainty. For standard, unambiguous prompts (\texttt{OR-Bench-80k}) and explicitly malicious requests (\texttt{OR-Bench-Toxic}), the particles largely agree on whether to comply or refuse, keeping the epistemic disagreement stable. However, on the \texttt{OR-Bench-Hard} split, which inherently stresses the blurry boundary between helpfulness and harmlessness, we observe a marked increase in the variance of refusal rates, $\operatorname{Var}(q_i)$ especially for $\alpha= 0.8$ compared to $\alpha=0$. 

\begin{table}[ht]
\centering
\caption{Mean refusal rates ($\bar{q}$) and epistemic disagreement ($\operatorname{Var}(q_i)$) on the OR benchmark splits across the $M=10$ LoRA particles. Format is $\bar{q}$ ($\operatorname{Var}(q_i)$).}
\label{tab:dpo_results}
\begin{tabular}{lcccc}
\toprule
\textbf{Model} & \textbf{Configuration} & \textbf{OR-Bench-80k} & \textbf{OR-Bench-Hard} & \textbf{OR-Bench-Toxic} \\
\midrule
\multirow{2}{*}{Phi-3-mini-4k} 
& $\alpha=0.0$       & 0.029 (0.006) & 0.232 (0.054) & 0.662 (0.050) \\
& $\alpha=0.8$ & 0.024 (0.005) & 0.297 (0.118) & 0.653 (0.052) \\
\midrule
\multirow{2}{*}{Qwen2-1.5B}    
& $\alpha=0.0$      & 0.004 ($3\times 10^{-4}$) & 0.142 (0.045) & 0.281 (0.020) \\
& $\alpha=0.8$ & 0.003 ($1\times 10^{-4}$) & 0.189 (0.096) & 0.256 (0.022) \\
\midrule
\multirow{2}{*}{Nemotron-3-8B} 
& $\alpha=0.0$       & 0.005 ($1\times10^{-5}$) & 0.446 (0.060) & 0.678 (0.047) \\
& $\alpha=0.8$ & 0.003 ($1\times10^{-5}$) & 0.485 (0.141) & 0.719 (0.046) \\
\bottomrule
\end{tabular}
\end{table}

\section{Discussion}

The $\alpha$-Rényi variational framework provides a principled method for learning distributions over post-training parameters, bridging the gap between classical variational Bayes and predictively oriented posterior learning. By treating a distribution of LoRA adapters as an interacting ensemble, the objective naturally induces a ``soft routing'' mechanism during training. Instead of forcing heterogeneous data and conflicting preferences into a single, compromised parameter vector, tuning the interpolation parameter $\alpha$ allows individual particles to specialise. 

Crucially, the degree of this specialisation does not need to be fixed in advance. The identification of a critical threshold, $\alpha_{\text{critical}}$, at which posterior spread becomes locally favourable under misspecification, points toward the dynamic selection of $\alpha$ as a powerful mechanism. 
Because this threshold can be numerically estimated (as noted in Section \ref{subsec:stability_misspecification}), $\alpha$ could be shifted from a static hyperparameter to an actively tuned variable. By adapting $\alpha$ during training in response to batch heterogeneity, the model could autonomously modulate its behaviour by enforcing tight generalisation on clean, unambiguous data, while dynamically increasing $\alpha$ to promote localised epistemic uncertainty when encountering highly conflicting or contaminated subsets. 

This capacity to represent unresolved uncertainty has profound implications for LLM alignment, particularly in paradigms like DPO. Alignment data inherently contains subjective ambiguity and competing pressures, such as the classic tension between remaining helpful and ensuring harmlessness. A single adapter forced to absorb these competing objectives often defaults to brittle compromises, such as over-refusing benign prompts to avoid catastrophic failures. By employing the $\alpha$-R\'enyi ensemble, conflicting alignment pressures are instead distributed across complementary specialists, representing ambiguity directly as posterior diversity.

\newpage
\section*{Acknowledgments}
PCE and GT gratefully acknowledges support from the EPSRC through the Centre for Doctoral Training in
Modern Statistics and Statistical Machine Learning (StatML), grant no. EP/S023151/1.  AD gratefully acknowledges support from AISI through the Alignment Project.  The authors are grateful for helpful discussions with Jiwon Park and Tom Coates.

\bibliographystyle{abbrv}
\bibliography{refs}

\newpage

\appendix

\section{Proofs}
\label{app:Proofs}
\begin{proof}[Proof of Lemma \ref{lem:alpha-interpolation}]
Let
\[
Z:=p_\theta(y\mid x),
\qquad
\theta\sim Q.
\]
By assumption, \(Z>0\) \(Q\)-a.s. The \(\alpha\)-loss may be written as
\[
\ell_\alpha(Q;x,y)
=
-\frac1\alpha \log \mathbb E_Q[Z^\alpha]
=
-\frac1\alpha \log \mathbb E_Q[\exp(\alpha\log Z)].
\]

\emph{Limit as \(\alpha\to0\).}
Define \(
\varphi(\alpha)
:=
\log \mathbb E_Q[\exp(\alpha\log Z)]\).
Since \(\varphi(0)=0\) and, under the stated integrability assumptions,
\(
\varphi'(0)
=
\mathbb E_Q[\log Z],
\)
we obtain
\[
\ell_\alpha(Q;x,y)
=
-\frac{\varphi(\alpha)}{\alpha}
\longrightarrow
-\varphi'(0)
=
-\mathbb E_Q[\log Z].
\]
This proves the first claim.

\emph{Endpoint \(\alpha=1\).}
At \(\alpha=1\),
\[
\ell_1(Q;x,y)
=
-\log \mathbb E_Q[Z]
=
-\log \int_\Theta p_\theta(y\mid x)\,Q(d\theta)
=
-\log p_Q(y\mid x).
\]

\emph{Interpolation inequality.}
For \(0<\alpha\le1\), the standard power-mean inequality for the positive random
variable \(Z\) gives
\[
\exp\{\mathbb E_Q[\log Z]\}
\le
\big(\mathbb E_Q[Z^\alpha]\big)^{1/\alpha}
\le
\mathbb E_Q[Z].
\]
Taking \(-\log\) throughout reverses the inequalities and yields
\[
-\log \mathbb E_Q[Z]
\le
-\frac1\alpha\log \mathbb E_Q[Z^\alpha]
\le
-\mathbb E_Q[\log Z],
\]
which is the desired interpolation inequality.

\emph{Monotonicity.}
Writing
\(\varphi(\alpha)
=
\log \mathbb E_Q[Z^\alpha]\),
then
\(
\ell_\alpha(Q;x,y)
=
-\frac{\varphi(\alpha)}{\alpha}
\).
For \(\alpha>0\), define the tilted distribution \(Q_\alpha\) by
\[
\frac{dQ_\alpha}{dQ}(\theta)
=
\frac{Z(\theta)^\alpha}{\mathbb E_Q[Z^\alpha]}.
\]
Differentiating \(\varphi\) gives
\[
\varphi'(\alpha)
=
\mathbb E_{Q_\alpha}[\log Z].
\]
Therefore
\[
\frac{d}{d\alpha}\ell_\alpha(Q;x,y)
=
\frac{\varphi(\alpha)-\alpha\varphi'(\alpha)}{\alpha^2}.
\]
On the other hand,
\[
\operatorname{KL}(Q_\alpha\|Q)
=
\mathbb E_{Q_\alpha}
\left[
\log \frac{dQ_\alpha}{dQ}
\right]
=
\mathbb E_{Q_\alpha}
\left[
\alpha\log Z-\varphi(\alpha)
\right]
=
\alpha\varphi'(\alpha)-\varphi(\alpha).
\]
Thus
\[
\frac{d}{d\alpha}\ell_\alpha(Q;x,y)
=
-\frac{1}{\alpha^2}
\operatorname{KL}(Q_\alpha\|Q)
\le 0.
\]
Hence \(\alpha\mapsto\ell_\alpha(Q;x,y)\) is non-increasing.

Finally, equality in the monotonicity derivative occurs if and only if
\(\operatorname{KL}(Q_\alpha\|Q)=0\), i.e. \(Q_\alpha=Q\). This holds if and
only if \(Z^\alpha\) is \(Q\)-a.s. constant, equivalently if
\(p_\theta(y\mid x)\) is \(Q\)-a.s. constant. This proves the strictness claim.
\end{proof}

\begin{proof}[Proof of Lemma \ref{lem:alpha_jensen_gap_variance_bound}]
Let
\[
Z:=p_\theta(y\mid x),
\qquad
\theta\sim Q,
\]
and define \(U:=Z^\alpha\). By assumption,
\[
0<U\le M_{x,y}^{\alpha}
\qquad Q\text{-a.s.}
\]
Moreover,
\[
\ell_0(Q;x,y)
=
-\mathbb E_Q[\log Z]
=
-\frac1\alpha \mathbb E_Q[\log U],
\]
and
\[
\ell_\alpha(Q;x,y)
=
-\frac1\alpha \log \mathbb E_Q[U].
\]
Therefore
\[
\ell_\alpha(Q;x,y)-\ell_0(Q;x,y)
=
\frac1\alpha
\left(
\mathbb E_Q[\log U]
-
\log \mathbb E_Q[U]
\right).
\]

The function \(f(u)=\log u\) is \(1/M_{x,y}^{2\alpha}\)-strongly concave on
\((0,M_{x,y}^{\alpha}]\), since
\[
f''(u)=-\frac1{u^2}
\le
-\frac1{M_{x,y}^{2\alpha}}.
\]
Thus, applying the strong-concavity Jensen gap bound to \(U\),
\[
\mathbb E_Q[\log U]
-
\log \mathbb E_Q[U]
\le
-
\frac{1}{2M_{x,y}^{2\alpha}}
\operatorname{Var}_Q(U).
\]
Substituting \(U=Z^\alpha=p_\theta(y\mid x)^\alpha\) gives
\[
\ell_\alpha(Q;x,y)-\ell_0(Q;x,y)
\le
-\frac{1}{2\alpha M_{x,y}^{2\alpha}}
\operatorname{Var}_{\theta\sim Q}
\!\left(
p_\theta(y\mid x)^\alpha
\right),
\]
thus proving the result.
\end{proof}

\begin{proof}[Proof of Lemma \ref{lem:convexity_alpha_objective}]
Let \(Q_0,Q_1\in\mathcal P(\Theta)\), let \(t\in[0,1]\), and define
\[
Q_t:=(1-t)Q_0+tQ_1.
\]
For a fixed data point \(z=(x,y)\), linearity of integration gives
\[
A_z(Q_t)
=
(1-t)A_z(Q_0)+tA_z(Q_1).
\]
Since \(u\mapsto -\alpha^{-1}\log u\) is convex on \((0,\infty)\), it follows that
\[
\ell_\alpha(Q_t;z)
\le
(1-t)\ell_\alpha(Q_0;z)
+
t\ell_\alpha(Q_1;z).
\]
Thus \(Q\mapsto \ell_\alpha(Q;z)\) is convex, and the finite-sample data term,
being a sum of convex functionals, is convex.

The KL term is convex in its first argument, so \(\mathcal F_{\alpha}\) is
convex. To obtain strong convexity, assume \(Q_0,Q_1\ll\pi_0\). Then
\[
(1-t)\operatorname{KL}(Q_0\|\pi_0)
+
t\operatorname{KL}(Q_1\|\pi_0)
-
\operatorname{KL}(Q_t\|\pi_0)
\]
equals
\[
(1-t)\operatorname{KL}(Q_0\|Q_t)
+
t\operatorname{KL}(Q_1\|Q_t).
\]
By Pinsker's inequality, using
\[
\operatorname{TV}(Q_0,Q_t)=t\,\operatorname{TV}(Q_0,Q_1),
\qquad
\operatorname{TV}(Q_1,Q_t)=(1-t)\operatorname{TV}(Q_0,Q_1),
\]
we get
\[
(1-t)\operatorname{KL}(Q_0\|Q_t)
+
t\operatorname{KL}(Q_1\|Q_t)
\ge
2t(1-t)\operatorname{TV}(Q_0,Q_1)^2.
\]
Therefore
\[
\operatorname{KL}(Q_t\|\pi_0)
\le
(1-t)\operatorname{KL}(Q_0\|\pi_0)
+
t\operatorname{KL}(Q_1\|\pi_0)
-
2t(1-t)\operatorname{TV}(Q_0,Q_1)^2.
\]
Combining this with convexity of the data term yields
\[
\mathcal F_{\alpha}(Q_t)
\le
(1-t)\mathcal F_{\alpha}(Q_0)
+
t\mathcal F_{\alpha}(Q_1)
-
2\lambda t(1-t)\operatorname{TV}(Q_0,Q_1)^2.
\]
Thus \(\mathcal F_{\alpha}\) is strongly convex in total variation when
\(\lambda>0\). In particular, it can have at most one minimiser.
\end{proof}

\begin{proof}[Proof of Proposition \ref{prop:self_consistency_alpha}]
By Lemma~\ref{lem:convexity_alpha_objective}, the functional
\(\mathcal F_{\alpha}\) is strongly convex on its effective domain when
\(\lambda>0\). Hence it has at most one minimiser. It remains to derive the
first-order condition.

Let
\[
A_i(q)
:=
\int_\Theta
p_\theta(y_i\mid x_i)^\alpha
q(\theta)\,\pi_0(d\theta).
\]
For \(Q(d\theta)=q(\theta)\pi_0(d\theta)\), the objective can be written as
\[
\mathcal F_{\alpha}(q)
=
-\frac1\alpha
\sum_{i=1}^N
\log A_i(q)
+
\lambda
\int_\Theta q(\theta)\log q(\theta)\,\pi_0(d\theta),
\]
subject to the constraint
\[
\int_\Theta q(\theta)\,\pi_0(d\theta)=1.
\]

Let \(h\) be any signed perturbation satisfying
\[
\int_\Theta h(\theta)\,\pi_0(d\theta)=0.
\]
For sufficiently small \(\varepsilon\), set \(q_\varepsilon=q+\varepsilon h\).
Differentiating at \(\varepsilon=0\) gives
\[
\left.
\frac{d}{d\varepsilon}
\mathcal F_{\alpha}(q_\varepsilon)
\right|_{\varepsilon=0}
=
\int_\Theta
\left[
\lambda(\log q(\theta)+1)
-
\frac1\alpha
\sum_{i=1}^N
\frac{p_\theta(y_i\mid x_i)^\alpha}{A_i(q)}
\right]
h(\theta)\,\pi_0(d\theta).
\]
At an interior minimiser \(q^\star_\alpha\), this derivative must vanish for
all such mass-preserving perturbations \(h\). Therefore the bracketed quantity
must be constant \(\pi_0\)-a.e.; that is, there exists a constant \(C\) such that
\[
\lambda(\log q^\star_\alpha(\theta)+1)
-
\frac1\alpha
\sum_{i=1}^N
\frac{p_\theta(y_i\mid x_i)^\alpha}{A_i(q^\star_\alpha)}
=
C.
\]
Rearranging,
\[
\log q^\star_\alpha(\theta)
=
C'
+
\frac{1}{\lambda\alpha}
\sum_{i=1}^N
\frac{p_\theta(y_i\mid x_i)^\alpha}{A_i(q^\star_\alpha)}.
\]
Exponentiating and absorbing \(e^{C'}\) into a normalising constant gives
\[
q^\star_\alpha(\theta)
=
\frac{1}{Z_\alpha}
\exp\!\left(
\frac{1}{\lambda\alpha}
\sum_{i=1}^N
\frac{p_\theta(y_i\mid x_i)^\alpha}{A_i(q^\star_\alpha)}
\right).
\]
Since
\[
A_i(q^\star_\alpha)
=
\int_\Theta
p_{\vartheta}(y_i\mid x_i)^\alpha
Q^\star_\alpha(d\vartheta),
\]
this is precisely the stated self-consistency equation.

Conversely, suppose \(q\) is a strictly positive density satisfying the
fixed-point equation. Then the preceding calculation shows that the first
variation of \(\mathcal F_{\alpha}\) at \(q\) vanishes in every
mass-preserving direction. Since \(\mathcal F_{\alpha}\) is convex, \(q\) is a
global minimiser. Since the functional is strongly convex when \(\lambda>0\),
this minimiser is unique.
\end{proof}

\begin{proof}[Proof of Proposition \ref{prop:well_specified_minimizers}]
    First, we show that if the model is well-specified, then $\delta_{\theta^\star}$ is a minimiser of the population risk $\mathcal{R}_\alpha(Q)$. Using Lemma \ref{lem:alpha-interpolation}, we have the following bound
    \begin{equation*}
        \mathcal{R}_\alpha(Q) = \mathbb{E}_{P^\star} [\ell_\alpha(Q; X, Y)] \geq \mathbb{E}_{P^\star} [\ell_1(Q; X, Y)]= \mathbb{E}_{P^*}[-\log p_Q(Y|X)] = \mathcal{R}_1(Q).
    \end{equation*}
    Using the non-negativity of the KL divergence, the cross-entropy between the true conditional distribution $p^*(y|x)$ and any other distribution $p_Q(y|x)$ is strictly minimised when the two distributions are identical. Therefore,
    $$\mathbb{E}_{Y \sim p^*(\cdot|X)} [-\log p_Q(Y|X)] \ge \mathbb{E}_{Y \sim p^*(\cdot|X)} [-\log p^*(Y|X)].$$
    Taking the expectation over $X$, this gives us a lower bound of the population risk $\mathcal{R}_\alpha(Q)$ for any $Q$
    $$\mathcal{R}_\alpha(Q) \ge \mathcal{R}_1(Q) \ge \mathbb{E}_{P^*}[-\log p^*(Y|X)].$$
    The lower bound $\mathbb{E}_{P^*}[-\log p^*(Y|X)]$ is attained when $Q = \delta_{\theta^*}$ centred at the true parameter. Therefore, $\delta_{\theta^*}$ is a global minimiser of the population risk for every $\alpha \in (0, 1]$.
    
Now we characterise the minimisers. Any minimiser $Q$ must achieve this lower bound, requiring both inequalities above to hold with equality. The second inequality is an equality if and only if $\mathcal{R}_1(Q) = \mathbb{E}_{P^\star}[-\log p^\star(Y|X)]$, which implies the KL divergence between $p^\star(\cdot|x)$ and $p_Q(\cdot|x)$ is zero for $P^\star$-a.e. $x$. Thus, for every $\alpha \in (0, 1]$, any minimiser must satisfy
$$p_Q(\cdot|x) = p^\star(\cdot|x) \quad \text{for } P^\star\text{-a.e. } x.$$
This establishes the claim for $\alpha=1$.

For $0 < \alpha < 1$, achieving the lower bound additionally requires the first inequality to hold with equality: $\mathcal{R}_\alpha(Q) = \mathcal{R}_1(Q)$. Because $\ell_\alpha(Q; x, y) \ge \ell_1(Q; x, y)$ pointwise, equality in expectation requires $\ell_\alpha(Q; X, Y) = \ell_1(Q; X, Y)$ almost surely. The bound in Lemma \ref{lem:alpha-interpolation} relies on Jensen's inequality applied to the strictly concave function $t \mapsto t^\alpha$. Therefore, equality holds if and only if the random variable is constant almost surely. 
Thus, $p_\theta(\cdot|x)$ must be constant for $Q$-a.e. $\theta$. In particular, this constant must be exactly $p^\star(\cdot|x)$. Therefore, for $Q$-a.e. $\theta$:
$$p_\theta(\cdot|x) = p^\star(\cdot|x) \quad \text{for } P^\star\text{-a.e. } x.$$
This means that $Q$ must be supported entirely on the exact-fit set $\Theta^\star$.

\end{proof}

\begin{proof}[Proof of Proposition \ref{prop:local_dirac_expansion}]
    Let the negative log-likelihood be $L_\theta(z) = -\log p_\theta(y|x)$. We want to expand the per-example $\alpha$-loss $l_\alpha(Q; z)$ for a distribution $Q$ with mean $\theta$ and covariance matrix $\Sigma \to 0$. The $\alpha$-loss can be rewritten as
    \begin{equation*}
    l_\alpha(Q; z) = -\frac{1}{\alpha}\log\mathbb{E}_{\theta'\sim Q}[\exp(-\alpha L_{\theta'}(z))].
    \end{equation*}
    Let $f(\theta') = \exp(-\alpha L_{\theta'}(z))$. Taking a second-order Taylor expansion around the mean $\theta$, we have
    \[
    f(\theta') = f(\theta) + \nabla f(\theta)^\top(\theta' - \theta) + \tfrac{1}{2}(\theta' - \theta)^\top\nabla^2 f(\theta)(\theta' - \theta) + o(||\theta' - \theta||^2).
    \]
    By taking expectations, the first term vanishes resulting in
    \begin{equation*}
    \mathbb{E}_Q[f(\theta')] = f(\theta) + \frac{1}{2}\operatorname{Tr}(\nabla^2 f(\theta)\Sigma) + o(||\Sigma||).
    \end{equation*}
    Expanding the Hessian $\nabla f(\theta)$, it follows
    \begin{equation*}
    \mathbb{E}_Q[f(\theta')] = f(\theta)\left(1 + \frac{1}{2}\operatorname{Tr}\left(\left[\alpha^2\nabla L_\theta(z)\nabla L_\theta(z)^\top - \alpha\nabla^2 L_\theta(z)\right]\Sigma\right)\right) + o(||\Sigma||).
    \end{equation*}
    This leads to 
    $$l_\alpha(Q; z) = -\frac{1}{\alpha}\log\left(f(\theta)\left(1 + \frac{1}{2}\operatorname{Tr}\left(\left[\alpha^2\nabla L_\theta(z)\nabla L_\theta(z)^\top - \alpha\nabla^2 L_\theta(z)\right]\Sigma\right)\right) + o(||\Sigma||)\right).$$
    Using a first-order Taylor approximation $\log(1 + x) \approx x$ for small $x$, together with the cyclic property of the trace, yields
    $$l_\alpha(Q; z) = L_\theta(z) + \frac{1}{2}\text{Tr}(\nabla^2 L_\theta(z)\Sigma) - \frac{\alpha}{2}\nabla\log p_\theta(z)^\top\Sigma\nabla\log p_\theta(z) + o(||\Sigma||).$$
    For the population risk $\mathcal{R}_\alpha(Q) = \mathbb{E}_{P^*}[l_\alpha(Q; Z)]$, we take the expectation over the data-generating distribution $P^*$ and use $V(\theta) = \mathbb{E}_{P^*}[\nabla^2(-\log p_\theta(Z))]$, $J(\theta) = \mathbb{E}_{P^*}[\nabla\log p_\theta(Z)\nabla\log p_\theta(Z)^\top]$, which provides 
    $$\mathcal{R}_\alpha(Q) = \mathbb{E}_{P^*}[-\log p_\theta(Z)] + \frac{1}{2}\text{Tr}((V(\theta) - \alpha J(\theta))\Sigma) + o(||\Sigma||),$$
    concluding the proof.
\end{proof}

\section{\texorpdfstring{$\alpha$}{alpha}-R\'enyi ensemble training for direct preference optimisation}
\label{app:dpo_variant}

This appendix provides the full derivation of the preference optimisation variant of the \(\alpha\)-R\'enyi ensemble objective introduced in Section \ref{subsec:dpo_variant}.

A standard probabilistic model for pairwise comparisons is the Bradley-Terry model. Given a latent utility function \(r(x,y)\), it assumes that the probability that response \(y^+\) is preferred to response \(y^-\) satisfies
\begin{equation}
\mathbb P(y^+ \succ y^- \mid x)
=
\frac{\exp(r(x,y^+))}{\exp(r(x,y^+))+\exp(r(x,y^-))}
=
\sigma\!\big(r(x,y^+)-r(x,y^-)\big),
\label{eq:bradley_terry}
\end{equation}
where \(\sigma(u)=1/(1+e^{-u})\) is the logistic sigmoid. The corresponding negative log-likelihood for one preference pair is
\[
-\log \sigma\!\big(r(x,y^+)-r(x,y^-)\big).
\]

In the DPO framework, the latent utility is identified, up to an additive constant, with a scaled log-density ratio between the policy and a fixed reference model
\begin{equation}
r_\theta(x,y)
=
\beta\Big(
\log p_\theta(y\mid x)-\log p_{\mathrm{ref}}(y\mid x)
\Big),
\label{eq:dpo_reward}
\end{equation}
where \(\beta>0\) is an inverse-temperature parameter and \(p_{\mathrm{ref}}\) is a frozen reference model. Substituting \eqref{eq:dpo_reward} into \eqref{eq:bradley_terry} yields the standard DPO likelihood
\begin{equation*}
\mathbb P_\theta(y^+ \succ y^- \mid x)
=
\sigma\!\left(
\beta\left[
\log \frac{p_\theta(y^+\mid x)}{p_\theta(y^-\mid x)}
-
\log \frac{p_{\mathrm{ref}}(y^+\mid x)}{p_{\mathrm{ref}}(y^-\mid x)}
\right]
\right).
\label{eq:dpo_likelihood}
\end{equation*}

In the LoRA ensemble setting, each particle \(\theta_i\) induces its own preference score. For a preference triple \((x,y^+,y^-)\), define the particle-wise preference margin
\begin{equation*}
\Delta_i(x,y^+,y^-)
:=
\log p_{\theta_i}(y^+\mid x)
-
\log p_{\theta_i}(y^-\mid x)
-
\log p_{\mathrm{ref}}(y^+\mid x)
+
\log p_{\mathrm{ref}}(y^-\mid x).
\label{eq:particle_preference_marginn}
\end{equation*}
Each sequence log-likelihood is the usual autoregressive teacher-forced sum over target tokens. The corresponding Bradley-Terry preference likelihood for particle \(i\) is
\begin{equation}
r_i(x,y^+,y^-)
:=
\sigma\!\big(\beta\,\Delta_i(x,y^+,y^-)\big).
\label{eq:particle_preference_likelihood}
\end{equation}

Substituting this into the per-example \(\alpha\)-R\'enyi preference loss (Eq. \ref{eq:alpha_preference_loss}) and expanding over a minibatch \(\mathcal B=\{(x_b,y_b^+,y_b^-)\}_{b=1}^B\), the finite-particle objective becomes
\begin{equation*}
\hat{\mathcal F}_{\alpha,\mathrm{pref}}^{(M)}(\theta_1,\dots,\theta_M)
=
\frac{N}{B}\sum_{b=1}^B
\ell_\alpha^{\mathrm{pref}}(Q^M;x_b,y_b^+,y_b^-)
+
\mathcal R_{\mathrm{prior}}(\theta_1,\dots,\theta_M).
\label{eq:minibatch_preference_objectivee}
\end{equation*}

The \(\alpha\)-objective induces responsibilities across particles. For each minibatch example \(b\), define
\begin{equation*}
w_{i,b}^{(\alpha)}
=
\frac{r_i(x_b,y_b^+,y_b^-)^\alpha}
{\sum_{j=1}^M r_j(x_b,y_b^+,y_b^-)^\alpha}.
\label{eq:preference_responsibilitiess}
\end{equation*}
The gradient of \(\hat{\mathcal F}_{\alpha,\mathrm{pref}}^{(M)}\) with respect to particle \(\theta_i\) takes the form
\begin{equation*}
g_i
=
-\frac{N}{B}\sum_{b=1}^B
w_{i,b}^{(\alpha)}
\nabla_{\theta_i}\log r_i(x_b,y_b^+,y_b^-)
+
\nabla_{\theta_i}\mathcal R_{\mathrm{prior}}.
\label{eq:preference_particle_gradientt}
\end{equation*}
Using \eqref{eq:particle_preference_likelihood}, we have
\[
\nabla_{\theta_i}\log r_i
=
\beta\big(1-r_i\big)\nabla_{\theta_i}\Delta_i,
\]
where
\begin{equation*}
\nabla_{\theta_i}\Delta_i
=
\nabla_{\theta_i}\log p_{\theta_i}(y_b^+\mid x_b)
-
\nabla_{\theta_i}\log p_{\theta_i}(y_b^-\mid x_b).
\label{eq:preference_margin_gradientt}
\end{equation*}
Therefore, the gradient step remains fully compatible with standard autoregressive LLM training.

\section{Implementation details for LoRA-based $\alpha$-R\'enyi ensembles}
\label{app:algorithm}

\begin{algorithm}[!h]
\small{
\caption{Direct AdamW training of an \(\alpha\)-R\'enyi LoRA ensemble}
\label{alg:alpha_renyi_lora_sft}
\begin{algorithmic}[1]
    \Require supervised dataset \(\mathcal D=\{(x_n,y_n)\}_{n=1}^N\), frozen base model \(W_0\), number of LoRA particles \(M\), parameter \(\alpha\in[0,1]\), prior regularisation weight \(\lambda\ge 0\), number of training steps \(T\), AdamW optimiser.
    \State Initialise LoRA particles \(\theta_1,\ldots,\theta_M\).
    \State Freeze base model parameters \(W_0\).
    \For{\(t=1,\ldots,T\)}
        \State Sample minibatch \(\mathcal B=\{(x_b,y_b)\}_{b=1}^B\subset\mathcal D\).
        \For{\(i=1,\ldots,M\)}
            \State Compute teacher-forced sequence log-likelihoods
            \[
            s_{i,b}
            :=
            \log p_{\theta_i}(y_b\mid x_b)
            =
            \sum_{\tau=1}^{T_b}
            \log p_{\theta_i}(y_{b,\tau}\mid x_b,y_{b,<\tau})
            \qquad b=1,\ldots,B.
            \]
        \EndFor
        \If{\(\alpha=0\)}
            \State Compute the average particle loss
            \[
            \mathcal L_{\mathrm{data}}
            =
            -\frac{1}{BM}
            \sum_{b=1}^B\sum_{i=1}^M s_{i,b}.
            \]
        \Else
            \State Compute the \(\alpha\)-R\'enyi data loss using log-sum-exp:
            \[
            \mathcal L_{\mathrm{data}}
            =
            -\frac{1}{B\alpha}
            \sum_{b=1}^B
            \left[
            \operatorname{logsumexp}_{i=1}^M(\alpha s_{i,b})
            -
            \log M
            \right].
            \]
        \EndIf
               \State Compute a finite-particle approximation to the prior/KL regulariser: \(\mathcal R\in\{\mathcal R_{\mathrm{prior}},\mathcal R_{\mathrm{KDE}}\}\).
        \Statex \quad \textbf{Option A: prior-potential surrogate}
        \[
        \mathcal R_{\mathrm{prior}}
        =
        -\frac{\lambda}{M}
        \sum_{i=1}^M \log q_0(\theta_i).
        \]
        \Statex \quad \textbf{Option B: smoothed empirical KL}
        \[
        q_M^\varepsilon(\theta)
        =
        \frac1M\sum_{j=1}^M K_\varepsilon(\theta-\theta_j),
        \qquad
        \mathcal R_{\mathrm{KDE}}
        =
        \frac{\lambda}{M}
        \sum_{i=1}^M
        \left[
        \log q_M^\varepsilon(\theta_i)
        -
        \log q_0(\theta_i)
        \right].
        \]
        
        \State Form the finite-particle objective
        \[
        \mathcal J_\alpha
        =
        \mathcal L_{\mathrm{data}}
        +
        \frac{1}{N}\mathcal R.
        \]
        \State Update \(\theta_1,\ldots,\theta_M\) by performing one AdamW step on \(\nabla_{\theta_{1:M}}\mathcal J_\alpha\).
    \EndFor
    \State \textbf{Output:} LoRA ensemble \(\{\theta_i\}_{i=1}^M\), posterior approximation \(Q^M=\frac1M\sum_{i=1}^M\delta_{\theta_i}\).
\end{algorithmic}}
\normalsize
\end{algorithm}

The full algorithm is provided in Algorithm \ref{alg:alpha_renyi_lora_sft}.

\subsection{Efficient computation with shared frozen backbones}

A central practical advantage of the LoRA ensemble setting is that all particles share the same frozen base model. This induces substantial computational structure.

First, the dominant parameter memory remains that of the single base model \(W_0\), only the low-rank increments \(\Delta W(\theta_i)\) are replicated across particles. If the LoRA rank \(r\) is small, the memory overhead of the ensemble scales roughly linearly in \(M r\), rather than in the full model dimension.

Second, the particle dimension can be partially vectorised. Given a minibatch \(\{(x_b,y_b)\}_{b=1}^B\), one may form an augmented batch indexed by both data point and particle, evaluate all \(M\) LoRA variants in parallel, and compute the matrix $ \big(\log p_{\theta_i}(y_b\mid x_b)\big)_{1\le i\le M,\;1\le b\le B}$.  From this matrix, the responsibilities  are obtained by a softmax across the particle dimension after multiplying by \(\alpha\). In practice, for autoregressive sequence losses, this is done by accumulating token-level log-probabilities over the target continuation.

Third, because all particles are small modifications of a common backbone, one can exploit implementation-level sharing. For example, all non-adapted layers are evaluated identically across particles, and only the adapted projections differ. When memory permits, the LoRA updates can be stacked and applied in parallel using batched low-rank matrix multiplications. When memory is more constrained, particles can be processed in chunks while still sharing the same frozen model weights.

\subsection{Code snippet}
To illustrate the simplicity of integrating our approach, we provide a code snippet below showing how to modify standard LoRA to support $\alpha$-Rényi ensemble training.

\section{Experimental details}\label{app:experimental_details}
We provide further details regarding our experimental setup. Furthermore, we include a code snippet that demonstrates the lightweight adaptation of our framework into standard LoRA layers.

\paragraph{Models} We evaluate our approach across three different base models: Phi-3-mini-4k-instruct, Qwen2-1.5B-Instruct, and Nemotron-3-8B-base-4k.

\subsection{Supervised fine-tuning experiment} 

\paragraph{Benchmark} 
We consider the MMLU benchmark \cite{hendrycks2021measuring_mmlu}, we split the data into training and test. Its high task diversity provides an ideal testbed for observing and validating the emergence of specialisation across the ensemble.

\paragraph{Evaluation}
We compute the test-set responsibilities, defined as
\begin{equation*}
w_{i}(x,y) = \frac{p_{\theta_i}(y\mid x)}{\sum_{j=1}^M p_{\theta_j}(y\mid x)}
\end{equation*}
and report them in Figure \ref{fig:heatmap_comparison} to illustrate the relative predictive strengths of each particle on a per-question basis.

\paragraph{Hyperparameters} 
All models are trained for 5 epochs. To determine the optimal configuration, we perform a hyperparameter sweep over learning rates $\eta \in \{1\times10^{-6}, 5\times10^{-6}, 1\times10^{-5}, 5\times10^{-5}, 1\times10^{-4}\}$ and LoRA ranks $r \in \{4, 8, 16, 32\}$. Based on this search, we set the learning rate to $5\times10^{-6}$ across all experiments. The optimal LoRA rank was found to be $r=4$ for Qwen2-1.5B and Nemotron-3-8B, and $r=8$ for Phi-3-mini-4k.

\includepdf[fitpaper=true, pages=-]{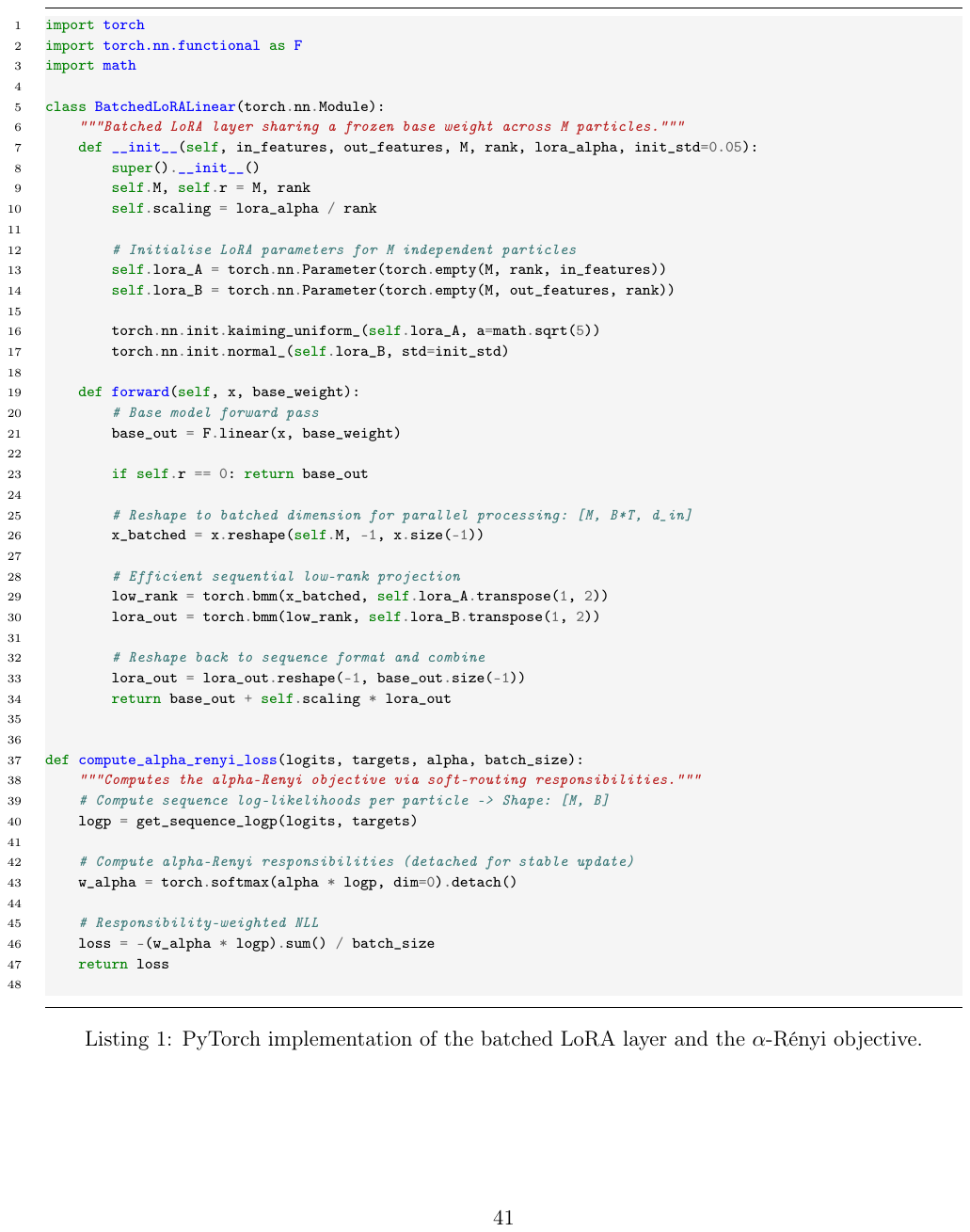}

\subsection{Direct preference optimisation experiment}

\paragraph{Training data}
Models are fine-tuned on the \texttt{trl-lib/ultrafeedback\_binarized} DPO dataset. To ensure high-quality preference signals, we filtered the dataset to include only pairs where the difference between the chosen and rejected scores strictly exceeded 1.

\paragraph{Benchmark} 
We use the OR benchmark \cite{cui2025or} to evaluate the ensemble's behaviour under ambiguity. 
We evaluated the models across three distinct splits: \texttt{OR-bench-80k, OR-bench-hard-1k} and \texttt{OR-bench-toxic}.

\paragraph{Evaluation} To compute particle-level refusal probabilities, we generated 12 responses per particle for each prompt.
We used an LLM-as-a-judge to decide whether the generated response is a refusal or not. We then report the mean and variance $\operatorname{Var}(q_i)$ of the refusal probabilities across the different particles.

\paragraph{Hyperparameters}
We performed a hyperparameter search similar to the one in the previous experiment. The ensemble of LoRA adapters was trained using the AdamW optimiser with a learning rate of $1\times 10^{-5}$ and a DPO inverse-temperature parameter $\beta = 0.1$.  In this experiment, we used an ensemble of 10 particles, with the interpolation parameter set to $\alpha = 0.8$ and a LoRA rank of 4 for all models.

\section{Related work}
\label{app:related_work}

Our work connects several strands of literature: Bayesian neural networks and variational inference, generalised Bayesian inference under misspecification, Rényi and entropic-risk objectives, predictively oriented posteriors, deep ensembles, parameter-efficient fine-tuning, and uncertainty-aware alignment. 

\paragraph{Bayesian neural networks and variational inference.}

Various works have sought to integrate the principles of Bayesian inference with deep learning models, effectively learning a conditional probability distribution over neural network parameters, permitting a principled handling of uncertainty and the incorporation of prior beliefs on the weights, through Bayes' theorem \citep{bayes1958essay}.   In theory, this  can provide a principled approach to quantifying epistemic uncertainty, with important applications in safety-sensitive settings, \cite{kendall2017uncertainties,maddox2019simple, blundell2015weight}, although this is highly dependent on an appropriate choice of prior for the weights   \citep{fortuinbayesian,cinquin2021pathologies}. 

Exact Bayesian inference is intractable for modern neural networks which has led to a body of approximate methods, including variational inference \citep{kingma2013auto,blundell2015weight,kingma2015variational}, Laplace approximations \citep{mackay1992practical,yang2024bayesian, deng2022accelerated}, sampling-based approaches \citep{neal2012bayesian,chen2014stochastic}, dropout-based approximations \citep{gal2016dropout}, and architecture-specific Bayesian approximations \citep{harrison2024variational}. Recent surveys provide a broad overview of this area \citep{arbel2026primer, papamarkouposition}.

There is a growing body of literature on introducing such  approximations for training Bayesian neural networks including transformer models and LLMs, as a means of calibrating confidence of model outputs, \cite{tran2019bayesian, jiang2021can, xue2021bayesian, fan2020bayesian, zhang2021bayesian, daxberger2021laplace}.  In \cite{xiao2022uncertainty} the authors perform a systematic empirical evaluation of Bayesian and ensemble methods for uncertainty quantification for LLMs, see also \cite{mittal2025context}.  One recent direction relevant to our work is Bayesian modelling over low rank perturbations of a frozen base model, \cite{yang2024bayesian, dusenberry2020efficient, doan2025bayesian, onal2024gaussian} effectively generalising Bayesian last-layer parameterisations \cite{tran2019bayesian}.     

\paragraph{Variational inference with a R\'{e}nyi loss.}

There is a substantial literature on replacing the Kullback-Leibler objective in variational inference by Rényi or \(\alpha\)-divergence objectives.   This has been motivated by the specific behaviours of standard VI which can induce severely under- or over- inflated posterior approximations, depending on the placement of the variational measure within the KL.   In \cite{hernandez2016black} and \cite{li2016renyi}, the authors study the use of Rényi $\alpha$-divergences as an objective in variational inference, yielding the Variational Rényi (VR) bound.    They demonstrated that the non-linear structure of the Rényi divergence allows the approximate posterior to smoothly interpolate between mass-covering ($\alpha \to -\infty$) and zero-forcing or mode-seeking ($\alpha \to +\infty$) behaviours.   The resulting VR objective is approximated through Monte Carlo approximations and the reparametrisation trick, allowing for the calculation of gradients that are scaled by an importance weight. However, this approach relies on updating the parameters of a single, fixed parametric distribution (such as a Gaussian).

To scale $\alpha$-divergence variational inference to deep neural networks, \cite{li2017dropout} proposed using Monte Carlo Dropout as the approximate posterior, demonstrating that the network could achieve improved, mass-covering uncertainty estimates. However, representing the posterior via dropout forces all ensemble members to share a single underlying weight matrix. Because the stochastic samples are merely binary masks applied to a shared set of parameters, the ensemble is not an appropriate mechanism to promote specialisation.

In \cite{yang2020alpha}, the authors analysed the statistical guarantees of a similar $\alpha$-VB framework. In their formulation, the expected log-likelihood is evaluated linearly under the variational distribution, making their method mathematically equivalent to finding the closest approximation to a fractional or tempered posterior. They established that for $\alpha \in (0, 1]$, the $\alpha$-VB posterior concentrates around the true data-generating parameter at the minimax optimal rate, providing rigorous frequentist guarantees for point estimation in latent variable models.

Optimising the non-linear $\alpha$-Rényi objective presents significant practical challenges, primarily because the objective lacks Euclidean smoothness and can cause standard gradient descent to become highly unstable. To resolve this, \cite{guilmeau2025regularized} proposed mapping the optimisation problem into a non-Euclidean geometry,  by using a Bregman Proximal Gradient algorithm induced by the log-partition function of an exponential family.   The resulting scheme yields a stable, relaxed moment-matching scheme with strict convergence guarantees, at least for fixed exponential variational families.

\paragraph{Entropic risk minimisation.}

The $\alpha$-Renyi variational objective relates to the notion of entropic risk, which is widely used in finance 
 and statistics, \citep{follmer2011entropic,pichler2020entropy}.  For a random variable $Z$ it is defined via the cumulant generating function
$$
	R(t) = \frac{1}{t}\log \mathbb{E}[e^{tZ}],
$$
where $t\in \mathbb{R}$ is a temperature parameter.    For a fixed datapoint $(x,y)$, we identify the per-example loss likelihood $l_{\alpha}(Q; x,y)$ via entropic risk by choosing $Z = -\log p_{\theta}(y \mid x)$, where $\theta \sim Q$  and $t= -\alpha$.

The entropic risk's behaviour as a risk measure is largely determined by $t$.  When $t > 0$ ($\alpha < 0$), the exponential weights the worst outcomes, i.e. the highest losses, so that the objective tries to minimise the maximum loss,  thus preventing any single data point from suffering in a highly risk averse manner.   When $t < 0$ ($\alpha > 0$), the exponential heavily weights the best outcomes (lowest losses). The objective is satisfied if at least some outcomes are very good, ignoring the bad ones.  In the limit $t\rightarrow -\infty$, the objective is governed by the absolute minimum loss.  This ``risk-seeking'' behaviour is the property we exploit in our formulation, as it promotes the emergence of specialists.    

In the context of machine learning, entropic risk measures have been studied through Tilted Entropic Risk Minimisation (TERM) \citep{litilted}.   TERM applies tilting across the empirical data distribution to dynamically reweight samples, demonstrating that negative tilt parameters successfully suppress the gradients of noisy outliers.   Rather than tilting the empirical data distribution for a single model, our approach tilts the parameter posterior for an ensemble. Because our positive $\alpha$ corresponds mathematically to a negative tilt in the TERM framework, our interacting particle system naturally inherits these outlier-suppression properties, routing clean data to specialised particles while starving corrupted data of gradient influence.

\paragraph{Bayesian inference under model misspecification.}

Standard Bayesian updating can behave poorly under model misspecification. When the assumed model class does not contain the data-generating distribution, the posterior may concentrate around a pseudo-true parameter that is optimal for the wrong objective, and the resulting posterior predictive can be overconfident or even inconsistent \citep{grunwald2017inconsistency}. This has motivated generalised Bayesian approaches, in which the likelihood is replaced by a loss or tempered by a learning-rate parameter \citep{grunwald2007suboptimal,bissiri2016general,Knoblauch2022RoT}. Gibbs posteriors \citep{jiang2008gibbs,martin2022direct}, fractional or tempered posteriors \citep{bhattacharya2019bayesian,wenzel2020good}, and safe-Bayesian methods \citep{grunwald2012safe} all modify the strength or form of posterior updating to improve robustness under misspecification.

The $\alpha$-R\'enyi objective does not fall within the the generalised Bayesian framework, due to its nonlinearity in the posterior distribution $Q$.   Thus the posterior is no longer simply an exponential tilt of the prior by an additive empirical loss; it is a self-consistent predictive object.  In particular, both approaches seek to address the same problem within classical Bayesian inference, namely model misspecification.

\paragraph{Predictively oriented posteriors and model averaging.}

Our work is most closely related conceptually to predictively oriented posteriors (PrO) \citep{McLatchie2025PrO}. In that framework, the distribution \(Q\) over parameters is chosen according to the predictive quality of the induced mixture
\[
p_Q(y\mid x)
=
\int p_\theta(y\mid x)\,Q(d\theta),
\]
rather than according to the average fit of individual parameter values. This shift is especially important under misspecification, where no single parameter may adequately explain the data, but a mixture of complementary predictors may perform well. Related ideas appear in Bayesian model averaging, PAC-Bayesian analyses of majority votes, and ensemble-risk bounds \citep{lacasse2006pac,germain2015risk,wu2021chebyshev,masegosa2020learning}.

Recent work has also proposed using the discrepancy between classical Bayesian posteriors and predictively oriented posteriors as a diagnostic for model misspecification \citep{Liu2025Misspecification}. Our goal is complementary: rather than comparing the two endpoints, we introduce a continuous family of objectives between them and study how intermediate values of \(\alpha\) can induce posterior spread, specialisation, and uncertainty.

\paragraph{Deep ensembles and diversity-promoting training.}

Deep ensembles provide a practical and widely used alternative to Bayesian neural networks \citep{fort2019deep}. By training several models independently and averaging their predictions, deep ensembles often produce strong empirical uncertainty estimates and improved robustness. They are simple, scalable, and compatible with modern deep-learning pipelines. Calibration methods, including conformal and post-hoc approaches, can further improve the reliability of their predictive uncertainty \citep{angelopoulos2025learn,rivera2024conformal}.

However, independently trained ensembles do not, by themselves, specify what distribution over models they approximate. Their diversity is induced indirectly through initialisation, data order, optimiser noise, or architectural variation. Several works have therefore introduced explicit interaction, repulsion, or posterior-matching mechanisms to make ensembles more closely resemble Bayesian posterior samples or coordinated predictive distributions \citep{d2021repulsive, wild2023rigorous}. Our method belongs to this broad family of coordinated ensembles, but differs in its variational objective. Ensemble members interact through \(\alpha\)-dependent responsibilities, so that each example is softly routed toward the particles that explain it well. Diversity is therefore tied directly to predictive specialisation rather than being imposed only through external repulsion.

This also distinguishes our approach from standard mixture-of-experts methods. In mixture-of-experts models, specialisation is usually driven by a learned gating network that routes inputs to experts. In our setting, there is no separate gating model. The \emph{responsibilities} $w_i^{(\alpha)}$ arise from the variational objective itself
\[
w_i^{(\alpha)}(x,y)
=
\frac{p_{\theta_i}(y\mid x)^\alpha}
{\sum_{j=1}^M p_{\theta_j}(y\mid x)^\alpha}.
\]
The routing is thus induced implicitly from the  likelihood, rather than learned as an additional  component.

\paragraph{Parameter-efficient fine-tuning and LoRA ensembles.}

Large language models are commonly adapted through parameter-efficient fine-tuning (PEFT), in which the base model is frozen and only a small number of additional or selected parameters are trained \citep{ding2022delta, han2024parameterefficient,xu2023parameter}. Methods include low-rank reparametrisations such as LoRA \citep{hu2022lora,yang2024bayesian,valipour2023dylora}, adapter modules \citep{houlsby2019parameter,pfeiffer2021adapterfusion}, prompt- and prefix-tuning \citep{li2023prefix, li2021prefix}, and selective fine-tuning methods \citep{lawton2023neural, guo2021parameter}.

PEFT is attractive for uncertainty-aware post-training because the adaptation space is much smaller than the full parameter space. A distribution over all transformer weights is usually impractical, but a distribution over LoRA adapters can be represented by a modest ensemble of particles. Our framework takes this route: each particle is a LoRA adapter attached to a shared frozen base model. The base weights are common across particles, while the low-rank updates are trained jointly through the \(\alpha\)-Rényi objective. This gives a scalable approximation to a posterior-like distribution over adaptations rather than over the entire model.

Existing PEFT methods typically produce a single adapted model. Even when multiple adapters are trained, they are often combined heuristically or selected for different tasks. Our contribution is to provide a variational objective for training the adapter ensemble as a single interacting distribution, with \(\alpha\) controlling the balance between shared generalisation and specialisation.

\paragraph{Preference optimisation and uncertainty in alignment.}

Post-training for alignment often relies on supervised fine-tuning followed by preference-based methods such as RLHF, DPO, or related objectives, \cite{christiano2017deep, ouyang2022training}. These methods usually optimise a single policy against preference data, either through an explicit reward model or through an implicit preference likelihood, \cite{bradley1952rank, rafailov2023direct, azar2024general} . In the presence of inconsistent preferences , ambiguous prompts, adversarial examples, or underspecified safety constraints, a single adapted model can collapse multiple plausible behaviours into one compromise \cite{bakker2022fine, casper2023open}. This can contribute to over-refusal, over-permissiveness, or brittle behaviour under distribution shift.

Our preference-learning variant extends the \(\alpha\)-Rényi objective to pairwise preference likelihoods. Each particle induces its own DPO-style preference probability, and the ensemble is trained through an \(\alpha\)-aggregated preference likelihood. As in the supervised case, the resulting responsibilities route preference examples toward the particles that currently explain them well. This gives a distributional generalisation of DPO: conflicting alignment pressures can be represented through posterior diversity rather than forced into a single adapter.

This distributional view is also useful for safety evaluation. Instead of evaluating only a single post-trained model, one can estimate disagreement across the learned adapter distribution and use this as an epistemic signal. High posterior disagreement may indicate prompts for which the alignment data do not determine a unique safe behaviour, suggesting a role for abstention, clarification, fallback policies, or targeted red-teaming \cite{houlsby2011bayesian, geifman2017selective, perez2022red}.

\end{document}